\documentclass[10pt]{article}
\usepackage{times}

\usepackage{microtype}
\usepackage{graphicx}
\usepackage{subfigure}
\usepackage{booktabs}

\usepackage{color}

\usepackage{algorithm}
\usepackage{algorithmic}

\usepackage{authblk}

\usepackage{tabu}
\usepackage{url}

\newlength{\defbaselineskip}
\usepackage{etoolbox}
\newtoggle{arxiv}
\toggletrue{arxiv}

\usepackage{cdesa-math}

\usepackage{mathtools}
\usepackage[dvipsnames]{xcolor}

\usepackage[numbers,sort]{natbib}

\usepackage{hyperref}

\newlength{\myfcwidth}
\begin{document}

\title{Minibatch Gibbs Sampling on Large Graphical Models}

\author[$\dagger$]{Christopher De Sa}
\author[$\dagger$]{Vincent Chen}
\author[$\ddagger$]{Wing Wong}
\affil[$\dagger$]{Department of Computer Science, Cornell University}
\affil[$\ddagger$]{Department of Statistics, Stanford University\vspace{4pt}}
\affil[ ]{\texttt{cdesa@cs.cornell.edu}, \texttt{zc346@cornell.edu}, \texttt{whwong@stanford.edu}}

\maketitle

\begin{abstract}
Gibbs sampling is the de facto Markov chain Monte Carlo method used for inference and learning on large scale graphical models.
For complicated factor graphs with lots of factors, the performance of Gibbs sampling can be limited by the computational cost of executing a single update step of the Markov chain.
This cost is proportional to the degree of the graph, the number of factors adjacent to each variable.
In this paper, we show how this cost can be reduced by using minibatching: subsampling the factors to form an estimate of their sum.
We introduce several minibatched variants of Gibbs, show that they can be made unbiased, prove bounds on their convergence rates, and show that under some conditions they can result in asymptotic single-update-run-time speedups over plain Gibbs sampling.

\end{abstract}

\section{Introduction}
\label{secIntro}

Gibbs sampling is a Markov chain Monte Carlo method that is one of the most widespread techniques used with graphical models~\cite{koller2009probabilistic}.
Gibbs sampling is an iterative method that repeatedly resamples a variable in the model from its conditional distribution, a process that is guaranteed to converge asymptotically to the desired distribution.
Since these updates are typically simple and fast to run, Gibbs sampling can be applied to a variety of problems, and has been used for inference on large-scale graphical models in many systems~\cite{newman2007distributed,lunn2009bugs,
mccallum2009factorie,smola2010architecture,NIPS2012_4832,zhang2014dimmwitted}.

Unfortunately, for large graphical models with many factors, the computational cost of running an iteration of Gibbs sampling can become prohibitive.
Even though Gibbs sampling is a \emph{graph-local} algorithm, in the sense that each update only needs to reference data associated with a local neighborhood of the factor graph, as graphs become large and highly connected, even these local neighborhoods can become huge.
To make matters worse, complicated models tend to have categorical variables with large domains: for example, the variables can represent US States or ZIP codes, where there are dozens to thousands of possible choices.
The cost of computing a single iteration of Gibbs sampling is proportional to the \emph{product} of the size of the local neighborhood and the size of the domain of the categorical random variables.
\cmd{More explicitly, if the \emph{maximum degree} of a variable in the factor graph is $\Delta$, and each variable can take on $D$ possible values, then a single Gibbs sampling update takes $O(D \Delta)$ time to compute in the worst case.}

\begin{table*}[t]
\caption{Single-iteration computational complexity bounds of our minibatching algorithms compared with ordinary Gibbs sampling, for parameter settings which can slow down convergence, as measured by the spectral gap, by no more than a $O(1)$ factor.}
\label{tabSummary}
\vskip 0.15in
\begin{center}
\begin{tabular}{lll}
\toprule
\textbf{Algorithm} & \textbf{Compute Cost/Iteration} & \textbf{Notes} \\
\midrule
Gibbs sampling & $O(D \Delta)$ & \\
\midrule
MIN-Gibbs: Minibatch Gibbs & $O(D \Psi^2)$ & with high probability \\
MGPMH: Minibatch-Gibbs-Proposal Metropolis Hastings & $O(D L^2 + \Delta)$ & \\
DoubleMIN-Gibbs: Doubly Minibatch Gibbs & $O(D L^2 + \Psi^2)$ & with high probability \\
\bottomrule
\end{tabular}
\end{center}
\vskip -0.1in
\end{table*}

This effect limits the scalability of Gibbs sampling.
What started as a fast, efficient system can become dramatically slower over the course of development as models become more complicated, new factors are added, and new values are included in the variables' domains.
To address this, practitioners have developed techniques such as \emph{pruning}~\cite{rekatsinas2017holoclean} which improve scalability by making the factor graphs simpler.
While this does make Gibbs sampling faster, it comes at the cost of introducing bias which reduces fidelity to the original model.

In this paper, we explore a more principled approach to making Gibbs sampling scalable.
Our approach is inspired by \emph{minibatching} in stochastic gradient descent, which has been used with great success to scale up machine learning training.
The main idea is that, when a Gibbs sampling update depends on a large local graph neighborhood, we can instead \emph{randomly subsample} the neighborhood.
Supposing the random sample is representative of the rest of the neighborhood, we can proceed to perform the update using just the random sample, rather than the (much larger) entire neighborhood.
In this paper, we study minibatching for Gibbs Sampling, and we make the following contributions:
\begin{itemize}
  \itemsep2pt 
  \item We introduce techniques for \emph{minibatching} which can dramatically reduce the cost of running Gibbs sampling, while provably adding \emph{no bias} to the samples.
  \item We prove bounds on the convergence rates of our minibatch Gibbs algorithms, as measured by the \emph{spectral gap} of the Markov chain. We give a recipe for how to set the algorithms' parameters so that the spectral gap of minibatch Gibbs can be bounded to be arbitrarily close to the spectral gap of the original Gibbs chain.
  \item For a class of graphs with bounded energy, we show doing this results in an asymptotic computation cost of $O(D + \Delta)$, a substantial speedup from the cost of Gibbs sampling which is $O(D \Delta)$.
\end{itemize}

\subsection{Background and Definitions}
\label{secSetup}

First, to make our claims rigorous, we describe our setting by defining a factor graph and the various conditions and conventions we will be using in the paper.
A factor graph~\cite{koller2009probabilistic} with $n$ variables determines a probability distribution $\pi$ over a state space $\Omega$.
In general, a state $x \in \Omega$ is an assignment of a value to each of the $n$ variables, and each variable has its own domain over which it can take on values.
For simplicity, in this paper we will assume all variables take on values in the same domain $\{1, \ldots, D\}$ for some constant $D$, which would make $\Omega$ the set of functions $x: \{1,\ldots,n\} \rightarrow \{1,\ldots,D\}$.\footnote{\cmd{This means that the state space $\Omega$ is discrete, and we will be focusing on discrete-valued models in this paper. Continuous-valued models could be approximated with our algorithms by discretizing their domains to any desired level of accuracy.}}
The probability distribution $\pi$ is determined as the product of some set of factors $\Phi$, such that for any $x \in \Omega$,
\[
  \textstyle
  \pi(x) \propto \exp\left( \sum_{\phi \in \Phi} \phi(x) \right);
\]
this is sometimes called the \emph{Gibbs measure}. In this document, we use the notation $\rho(v) \propto \exp(\epsilon_v)$ to mean the unique distribution that is proportional to $\exp(\epsilon_v)$, that is,
\[
  \rho(v) = \frac{\exp(\epsilon_v)}{\sum_{w = 1}^D \exp(\epsilon_w)}.
\]
Equivalently, we can define an \emph{energy function} $\zeta$, where
\[
  \textstyle
  \zeta(x) = \sum_{\phi \in \Phi} \phi(x),
\]
and let $\pi(x) \propto \exp(\zeta(x))$.
\cmd{(Note that this definition implicitly excludes models with hard constraints or zero-probability states.)}
This is called a factor graph because the factors $\phi$ and variables $i$ have a dependency relation in that each $\phi$ depends only on the values of some of the variables, but typically not all $n$.
Formally, a factor $\phi$ depends on a variable if changing the value of that variable could change the value of $\phi$---meaning, if there exists states $x$ and $y$ which differ only in variable $i$ and for which $\phi(x) \ne \phi(y)$.
Using this, we define the relation
\[
  A = \left\{ (i, \phi) \middle| \text{factor } \phi \text{ depends on variable } i \right\}.
\]
Equivalently, this is a bipartite graph on the variables and factors; this edge relation, together with the factor function descriptions, defines a factor graph.
Using this, we formally describe Gibbs sampling, which we write in Algorithm~\ref{algGibbs}.
Note that we use the notation $x(i) \leftarrow u$ to denote variable $i$ within state $x$ being reassigned the value $u$, and following standard relation notation we let $A[i] = \{ \phi | (i, \phi) \in A \}$ be the set of factors that depend on variable $i$.

\begin{algorithm}[t]
  \caption{Gibbs sampling}
    \begin{algorithmic}
    \label{algGibbs}
    \STATE \textbf{given:} initial state $x \in \Omega$
    \LOOP
      \STATE \textbf{sample} variable index $i$ uniformly from $\{1,\ldots,n\}$
      \FORALL{$u$ \textbf{in} $\{1, \ldots, D\}$}
        \STATE $x(i) \leftarrow u$
        \STATE $\epsilon_u \leftarrow \sum_{\gamma \in A[i]} \phi(x)$
      \ENDFOR
      \STATE construct distribution $\rho$ over $\{1, \ldots, D\}$ where
      \[
        \textstyle
        \rho(v) \propto \exp(\epsilon_v)
      \]
      \STATE \textbf{sample} $v$ from $\rho$.
      \STATE $x(i) \leftarrow v$
      \STATE \textbf{output sample} $x$
    \ENDLOOP
  \end{algorithmic}
  \end{algorithm}

Next, we describe several conditions on the factor functions that we will be using throughout this paper.
For all that follows, we will suppose that $\phi(x) \ge 0$ for all $\phi$ and for all $x$; this holds without \cmd{loss of generality (for models without hard constraints)} because adding a constant to a factor function $\phi$ does not change the resulting distribution.
\begin{definition}[Factor graph conditions]
  \label{defnFGC}
  The \emph{maximum energy} of a factor $\phi$ is the smallest $M_{\phi} \in \R$ such that,
  \[
    0 \le \phi(x) \le M_{\phi} \text{  for all $x \in \Omega$}.
  \]
  The \emph{total maximum energy} $\Psi$ of a graph is the sum of all the maximum energies
  \[
    \textstyle
    \Psi = \sum_{\phi \in \Phi} M_{\phi}.
  \]
  The \emph{local maximum energy} $L$ of a graph is the largest sum of all the maximum energies of factors that depend on a single variable $i$,
  \[
    \textstyle
    L = \max_{i \in \{1, \ldots, n\}} \sum_{\phi \in A[i]} M_{\phi}.
  \]
  The \emph{maximum degree} $\Delta$ of a factor graph is the largest number of factors that are connected to a variable,
  \[
    \textstyle
    \Delta = \max_{i \in \{1, \ldots, n\}} \Abs{ A[i] }.
  \]
\end{definition}
For large models with many low-energy factors, the total maximum energy can be much smaller than the number of factors, and the local maximum energy can be much smaller than the maximum degree.
We will introduce several minibatch Gibbs variants in this paper, and their relative computation costs, in terms of the quantities defined in Definition~\ref{defnFGC}, are summarized in Table~\ref{tabSummary}.

\subsection{Related Work}
\label{secRelatedWork}

Several recent papers have explored applying minibatching to speed up large-scale Markov chain Monte Carlo methods.
Our work is inspired by \citet{li2017mini}, which shows how minibatching can be applied to the Metropolis-Hastings algorithm~\cite{hastings1970monte}.
They show that their method, MINT-MCMC, can outperform existing techniques that use gradient information, such as stochastic gradient descent and stochastic gradient Langevin dynamics, on large-scale Bayesian neural network tasks.
This is one among several papers that have also showed how to use minibatching to speed up Metropolis-Hastings under various conditions \cite{korattikara2014austerity,bardenet2014towards,chen2016efficient}.
\cmd{More general related approaches include Firefly MC, which implements a minibatching-like computational pattern using auxiliary variables \cite{maclaurin2014firefly}, and block-Poisson estimation, which constructs unbiased log-likelihood estimates using Poisson variables~\cite{quiroz2018mcmc}.}
Our results are distinct from this line of work in that we are the first to study minibatch Gibbs sampling in depth\footnote{\citet{
johndrow2015approximations} does evaluate a version of a minibatched Gibbs sampler on a particular application, but does not study Gibbs in depth.}, and we are the first to address the effect of categorical random variables with large domains that can limit computational tractability.
Additionally, none of these papers provides any theoretical bounds on the convergence rates of their minibatched algorithm compared to the original MCMC chain.  

\section{MIN-Gibbs}

\begin{algorithm}[t]
  \caption{MIN-Gibbs: Minibatch Gibbs sampling}
    \begin{algorithmic}
    \label{algMINTGibbs}
    \STATE \textbf{given:} minibatch estimator distributions $\mu_x$ for $x \in \Omega$
    \STATE \textbf{given:} initial state $(x, \epsilon) \in \Omega \times \R$
    \LOOP
      \STATE \textbf{sample} variable index $i$ uniformly from $\{1,\ldots,n\}$
      \STATE $\epsilon_{x(i)} \leftarrow \epsilon$
      \FORALL{$u$ \textbf{in} $\{1, \ldots, D\} \setminus \{x(i)\}$}
        \STATE $x(i) \leftarrow u$
        \STATE \textbf{sample} energy $\epsilon_u$ from $\mu_x$
      \ENDFOR
      \STATE construct distribution $\rho$ over $\{1, \ldots, D\}$ where
      \[
        \rho(v) \propto \exp(\epsilon_v)
      \]
      \STATE \textbf{sample} $v$ from $\rho$.
      \STATE $x(i) \leftarrow v$
      \STATE $\epsilon \leftarrow \epsilon_v$
      \STATE \textbf{output sample} $(x, \epsilon)$
    \ENDLOOP
  \end{algorithmic}
  \end{algorithm}

In this section, we will present our first algorithm for applying minibatching to Gibbs sampling.
As there are many options when doing minibatching\footnote{These options include: what the batch size is, whether the batch size is fixed or random, whether to do weighted sampling, whether to sample with replacement, etc.} we will present our algorithm in terms of a general framework for Gibbs sampling in which we use an estimate for the energy rather than computing the energy exactly.
Specifically, we imagine the result of minibatching on some state $x \in \Omega$ is a random variable $\epsilon_x$ such that
\[
    \epsilon_x \approx \sum_{\phi \in \Phi} \phi(x).
\]
For example, we could assign $\epsilon_x$ as
\[
    \epsilon_x = \frac{|\Phi|}{B} \sum_{\phi \in S} \phi(x)
\]
where $S$ is a randomly chosen subset of $\Phi$ of size $B$.
More formally, we let $\mu_x$ be the distribution of $\epsilon_x$, and we assume our algorithm will have access to $\mu_x$ for every state $x$ and can draw samples from it.
For simplicity, we assume $\mu_x$ has finite support, which is true for any minibatch estimator.

We can now replace the exact sums in the Gibbs sampling algorithm with our approximations $\epsilon_x$.
If sampling from $\mu_x$ is easier than computing this sum, this will result in a faster algorithm than vanilla Gibbs sampling.
There is one further optimization: rather than re-estimating the energy of the current state $x$ at each iteration of our algorithm, instead we can \emph{cache} its value as it was computed in the previous step.
Doing this results in Algorithm~\ref{algMINTGibbs}, MIN-Gibbs.
We call our algorithm MIN-Gibbs because it was inspired by the Mini-batch Tempered MCMC (MINT-MCMC) algorithm presented by \citet{li2017mini} for applying minibatching to Metropolis-Hastings, another popular MCMC algorithm.
MINT-MCMC also used the idea of caching the energy to form an augmented system with state space $\Omega \times \R$.
Compared with their algorithm, our contributions are that: (1) we modify the technique to apply to Gibbs sampling; (2) we show that \emph{tempering} (which is sampling from a modified distribution at a higher temperature) and other sources of bias can be circumvented by using a bias-adjusted minibatching scheme; and (3) we prove bounds on the convergence rate of our technique.

Because it uses an energy estimate rather than the true energy, this algorithm is not equivalent to standard Gibbs sampling.
This raises the question: will it converge to the same distribution $\pi$, and if not, what will it converge to?
This question can be answered by using the property of \emph{reversibility} also known as \emph{detailed balance}.

\begin{definition}
  A Markov chain with transition matrix $T$ is \emph{reversible} if for some distribution $\bar \pi$ and all states $x, y \in \Omega$,
  \[
    \bar \pi(x) T(x, y) = \bar \pi(y) T(y,x).
  \]
\end{definition}

It is a well known result that if a chain $T$ is reversible, the distribution $\bar \pi$ will be a stationary distribution of $T$, that is, $\bar \pi T = \bar \pi$.
Thus we can use reversibility to determine the stationary distribution of a chain, and we do so for Algorithm~\ref{algMINTGibbs} in the following theorem.

\begin{theorem}
  \label{thmMINTreversible}
  The Markov chain described in Algorithm~\ref{algMINTGibbs} is reversible and has stationary distribution
  \[
    \bar \pi(x, \epsilon) \propto \mu_x(\epsilon) \cdot \exp(\epsilon)
  \]
  and its marginal stationary distribution in $x$ will be
  \[
    \bar \pi(x) \propto \Exv[\epsilon \sim \mu_x]{\exp(\epsilon)}.
  \]
\end{theorem}

Interestingly, this theorem shows that if we choose our estimation scheme such that for all $x \in \Omega$,
\begin{equation}
  \textstyle
  \label{eqnUnbiasedCondition}
  \Exv[\epsilon \sim \mu_x]{\exp(\epsilon)} = \exp( \zeta(x) ) = \prod_{\phi \in \Phi} \exp\left( \phi(x) \right)
\end{equation}
then MIN-Gibbs will actually be \emph{unbiased}: we will have $\pi(x) = \bar \pi(x)$.
Next, we will show how we can take advantage of this by constructing estimators that satisfy (\ref{eqnUnbiasedCondition}) using minibatching.
Suppose without loss of generality that each $\phi$ is non-negative, $\phi(x) \ge 0$.
Let $\lambda$ be a desired average minibatch size, and for each factor $\phi$, let $s_{\phi}$ be an independent Poisson-distributed random variable with mean $\lambda M_{\phi} / \Psi$.
Let $S \subset \Phi$ denote the set of factors $\phi$ for which $s_{\phi} > 0$.
Then for any state $x$, define the estimator
\begin{equation}
  \label{eqnUnbiasedEst}
  \epsilon_x = \sum_{\phi \in S} s_{\phi} \log\left( 1 + \frac{\Psi}{\lambda M_{\phi}} \phi(x) \right).
\end{equation}
\begin{lemma}
  \label{lemmaEstimatorUnbiased}
  The estimator $\epsilon_x$ defined in (\ref{eqnUnbiasedEst}) satisfies the unbiasedness condition in (\ref{eqnUnbiasedCondition}).
\end{lemma}
\begin{proof}
  The expected value of the exponential of the estimator defined in (\ref{eqnUnbiasedEst}) is
  \begin{align*}
    \textstyle
    \Exv{\exp(\epsilon_x)}
    =
    \Exv{\exp\left( \sum_{\phi \in S} s_{\phi} \log\left( 1 + \frac{\Psi}{\lambda M_{\phi}} \phi(x) \right) \right)}.
  \end{align*}
  Since the $s_{\phi}$ are independent, this becomes
  \begin{align*}
    \textstyle
    \Exv{\exp(\epsilon_x)}
    =
    \prod_{\phi \in \Phi} \Exv{\exp\left( s_{\phi} \log\left( 1 + \frac{\Psi}{\lambda M_{\phi}} \phi(x) \right) \right)}.
  \end{align*}
  Each of these constituent expected values is an evaluation of the moment generating function of a Poisson random variable. Applying the known expression for this MGF and simplifying gives us the expression in (\ref{eqnUnbiasedCondition}), which is what we wanted to prove.
\end{proof}
From the result of this lemma, we can see that using this \emph{bias-adjusted} estimator with MIN-Gibbs will result in an unbiased chain with stationary distribution $\pi$.
However, the chain being unbiased does not, by itself, mean that this method will be more effective than standard Gibbs.
For that to be true, it must also be the case that approximating the energy did not affect the convergence rate of the algorithm---at least not too much.
The convergence times of Gibbs samplers can vary dramatically, from time linear in the number of variables to exponential time.
As a result, it is plausible that approximating the energy as we do in Algorithm~\ref{algMINTGibbs} could switch us over from a fast-converging chain to a slow-converging one, and as a result even though the individual iterations would be computationally faster, the overall computation would be slow---or worse, would silently give incorrect answers when the chain fails to converge.
Recently, other methods that have been used to speed up Gibbs sampling, such as running asynchronously \cite{desa2016hoggibbs} and changing the order in which the variables are sampled \cite{he2016scan}, have been shown in some situations to result in algorithms that converge significantly slower than plain Gibbs sampling.
Because of this, if we want to use algorithms like MIN-Gibbs with confidence, we need to show that this will not happen for minibatching.

To show that minibatching does not have a disastrous effect on convergence, we need to bound the convergence rate of our algorithms.
To measure the convergence rate, we use a metric called the \emph{spectral gap}~\cite{levin2009markov}.
\begin{definition}
Let $T$ be the transition matrix of a reversible Markov chain.
Since it is reversible, its eigenvalues must all be real, and they can be ordered as
\[
  1 = \lambda_1 \ge \lambda_2 \ge \cdots \ge \lambda_{\Abs{\Omega}}.
\]
The spectral gap $\gamma$ is defined as $\gamma = \lambda_1 - \lambda_2$.
\end{definition}
The spectral gap measures the convergence rate of a Markov chain in terms of the $\ell_2$-distance, in that the larger the spectral gap, the faster the chain converges.
The spectral gap is related to several other metrics of convergence for Markov chains.
For example, it is a standard result that for a lazy Markov chain (a Markov chain is called lazy if at each iteration, it stays in its current state with probability at least $1/2$) the \emph{mixing time} of the chain, which measures the number of steps required to converge to within total-variation distance $\epsilon$ of the stationary distribution $\pi$, is bounded by
\[
  t_{\text{mix}}(\epsilon) \le \frac{1}{\gamma} \log\left( \frac{1}{\epsilon \cdot \min_{x \in \Omega} \pi(x)} \right).
\]
In order to determine the effect of using minibatching on convergence, we would like to bound the spectral gap of MIN-Gibbs in terms of the spectral gap of the original Gibbs sampling chain.
In the following theorem, we show that if the energy estimator $\epsilon_u$ is always sufficiently close to the true energy, then we can bound the spectral gap.
\begin{theorem}
  \label{thmMINTgap}
  Let $\bar \gamma$ be the spectral gap of MIN-Gibbs running with an energy estimator $\mu_x$ that has finite support and that satisfies, for some constant $\delta > 0$ and every $x \in \Omega$,
  \[
    \Prob[\epsilon_x \sim \mu_x]{\Abs{ \epsilon_x - \zeta(x) } \le \delta} = 1.
  \]
  Let $\gamma$ be the spectral gap of a vanilla Gibbs sampling chain running using the exact energy.
  Then,
  \[
    \bar \gamma
    \ge
    \exp(-6 \delta) \cdot \gamma.
  \]
\end{theorem}
That is, the convergence is slowed down by at most a constant factor of $\exp(-5 \delta)$---which is independent of the size of the problem.
This theorem guarantees that if we can restrict our estimates to being within a distance of $O(1)$ of the exact energy, then the convergence rate will not be slowed down by more than an $O(1)$ constant factor.

Unfortunately, the estimator presented in (\ref{eqnUnbiasedEst}) is not going to be always bounded by a constant distance from the exact energy. Still, \cmd{since it is the sum of independent terms with bounded variance}, we \emph{can} bound it with high probability.
\begin{lemma}
  \label{lemmaUnbiasedConcentration}
  For any constants $0 < \delta$ and $0 < a < 1$, if we assign an expected batch size
  \[
    \lambda
    \ge
    \max\left(
      \frac{
        8 \Psi^2
      }{
        \delta^2
      }
      \log\left( \frac{2}{a} \right)
      ,
      \frac{2 \Psi^2}{\delta}
    \right),
  \]
  then the estimator in $(\ref{eqnUnbiasedEst})$ satisfies
  $\Prob{ \Abs{\epsilon_x - \zeta(x) } \ge \delta} \le a$.
\end{lemma}
This lemma lets us construct minibatch estimators that remain arbitrarily close to the true energy with arbitrarily high probability.
Furthermore, we can do this with a minibatch size independent of the number of variables or factors, depending instead on the total energy.
This means that if we have a very large number of low-energy factors, we can get significant speedup from MIN-Gibbs with high-probability theoretical guarantees on the convergence rate.
In particular, since the computational cost of running a single epoch of MIN-Gibbs is $D$ times the minibatch size, and this lemma suggests we need to set $\lambda = \Omega(\Psi^2)$, it follows that the total computational cost of MIN-Gibbs will be $O(\Psi^2 D)$.

\paragraph{Validation of MIN-Gibbs.} Having characterized the effects of minibatching on Gibbs sampling, we present a synthetic scenario where Algorithm~\ref{algMINTGibbs} can be applied.
The Ising model \cite{ising1925beitrag} is a probabilistic model over an $N \times N$ lattice, with domain $x(i) \in \{-1,1\}$.
The Ising model has a physical interpretation in which each $x(i)$ represent the magnetic \emph{spin} at each site.
The energy of a configuration is given by:
\[
  \textstyle
  \zeta_{\text{Ising}}(x) = \sum_{i = 1}^N \sum_{j = 1}^N \beta \cdot A_{ij} \cdot (x(i) x(j) + 1),
\]
where $A_{ij}$ is called the \emph{interaction} between variable $i$ and $j$, and $\beta$ is the \emph{inverse temperature}.\footnote{In some settings, the Ising model is used with an additional bias term (which physically represents a static magnetic field), but here for simplicity we do not include this term.}
We chose to use the Ising model to validate MIN-Gibbs because it is a simple model that nevertheless has non-trivial statistical behavior.

We chose an Ising model in which each site is fully connected (i.e. $\Delta = N^2 - 1$),
and the strength of interaction $A_{ij}$ between any two sites is determined based on their distance by a Gaussian kernel.
We simulated Algorithm~\ref{algMINTGibbs} on a graph with $n = N^2 = 400$ and inverse temperature $\beta = 1$. \cmd{This $\beta$ parameter was hand-tuned such that the Gibbs sampler seemed to mix in about the number of iterations we wanted to run. To have a fair comparison, the same setup was then used to evaluate MIN-Gibbs.
For this model, $L = 2.21$ and $\Psi = 416.1$.}\footnote{Note that since we chose $\beta$ large enough that $\Psi^2 > \Delta$ for this model, we do not expect MIN-Gibbs to be faster than Gibbs sampling for this particular synthetic example.} 
We started the algorithm with a unmixed configuration where each site takes on the same state ($x(i) = 1$ for all $i$).

As the algorithm ran, we used the output samples to compute a running average of the marginal distributions of each variable.
By symmetry (since negating a state will not change its probability), the marginal distribution of each variable in the stationary distribution $\pi$ should be uniform, so we can use the distance between the estimated marginals and the uniform distribution as a proxy to evaluate the convergence of the Markov chain.
Figure~\ref{fig:MINT} shows the average $\ell_2$-distance error in the estimated marginals compared with the fully-mixed state.
Notice that as the batch size increases, MIN-Gibbs approaches vanilla Gibbs sampling.

\begin{figure}[!tb]
  \centering
  \includegraphics[width=\myfcwidth]{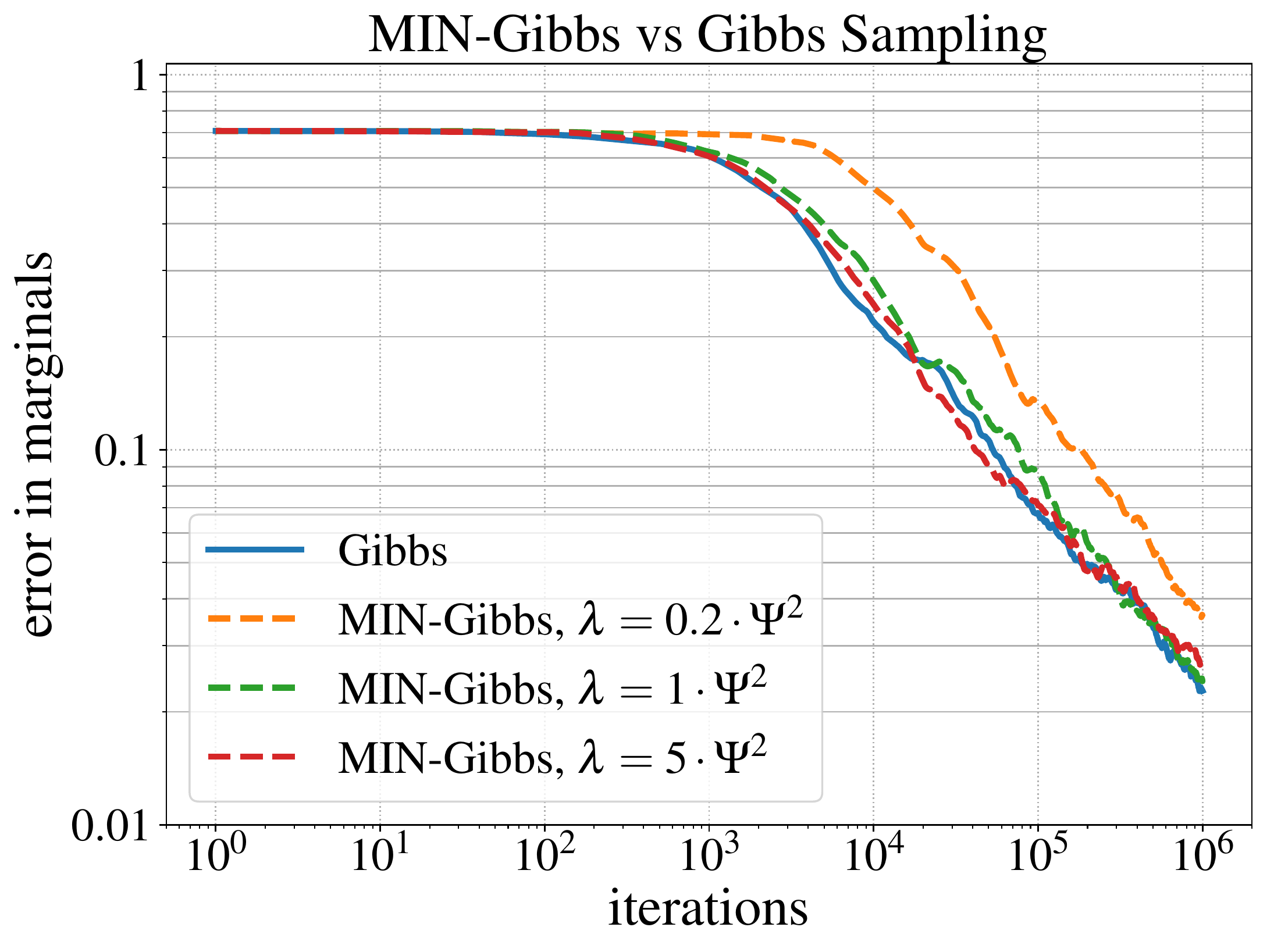}
    \caption{Convergence of marginal estimates for MIN-Gibbs compared with vanilla Gibbs sampling.}
  \vspace{-2ex}
  \label{fig:MINT}
  \end{figure}

\textbf{Using local structure.} Although MIN-Gibbs has a computational cost that is independent of the size of the graph, it still depends on the total maximum energy $\Psi$.
That is, unlike vanilla Gibbs sampling, which has a computational cost that depends only on the number of factors \emph{local} to the variable that is being sampled, MIN-Gibbs depends on \emph{global} properties of the graph.
This is because the estimators used by MIN-Gibbs are approximating the whole energy sum $\zeta(x)$, rather than approximating the sum over just those factors which depend on the variable that is being resampled. 
How can we modify MIN-Gibbs to take advantage of this local structure?

One straightforward way to do it is to allow the estimators $\epsilon_x$ to be \emph{dependent}, rather than independent.
That is, instead of choosing a unique minibatch every time we estimate the energy, we choose one fixed minibatch for each iteration.
Once our minibatch is fixed, the conditional distribution computation will exhibit the same cancellation of irrelevant factors that vanilla Gibbs has, and we will be able to compute the transition probabilities using only local information.
One side-effect of this is that we can no longer use the energy-caching technique that MIN-Gibbs uses, as we will no longer ever be estimating the total energy, but rather only the local component of the energy (the part that is dependent on the variable we are re-sampling).
This would result in something like Algorithm~\ref{algLocalMinibatchGibbs}.
(We could also consider variants of this algorithm that use different minibatching schemes, such as (\ref{eqnUnbiasedEst}), but here for simplicity we present the version that uses standard minibatching.)
Note that Algorithm~\ref{algLocalMinibatchGibbs} is nearly identical to vanilla Gibbs sampling, except that it uses a single minibatch to estimate all the energies in each iteration.

\begin{algorithm}[t]
  \caption{Local Minibatch Gibbs}
    \begin{algorithmic}
    \label{algLocalMinibatchGibbs}
    \STATE \textbf{given:} initial state $x \in \Omega \times \R$, minibatch size $B$
    \LOOP
      \STATE \textbf{sample} variable index $i$ uniformly from $\{1,\ldots,n\}$
      \STATE \textbf{sample} minibatch $S \subset A[i]$ uniformly s.t. $\Abs{S} = B$
      \FORALL{$u$ \textbf{in} $\{1, \ldots, D\}$}
        \STATE $x(i) \leftarrow u$
        \STATE $\epsilon_u \leftarrow \frac{\Abs{A[i]}}{\Abs{S}} \sum_{\gamma \in S} \phi(x)$
      \ENDFOR
      \STATE construct distribution $\rho$ over $\{1, \ldots, D\}$ where
      \[
        \rho(v) \propto \exp(\epsilon_v)
      \]
      \STATE \textbf{sample} $v$ from $\rho$.
      \STATE $x(i) \leftarrow v$
      \STATE \textbf{output sample} $x$
    \ENDLOOP
  \end{algorithmic}
  \end{algorithm}

Algorithm~\ref{algLocalMinibatchGibbs} will run iterations in time $O(B D)$, which can be substantially faster than plain Gibbs, which runs in $O(\Delta D)$.
We evaluate Algorithm~\ref{algLocalMinibatchGibbs} empirically, and demonstrate in Figure~\ref{fig:local} that it converges, with almost the same trajectory as plain Gibbs, for various values of the batch size $B$. The simulation here is on the same Ising model as in Algorithm~\ref{algMINTGibbs}, with the same parameters.
Unfortunately, it is unclear if there is anything useful we can say about its convergence rate or even what it converges to.
Unlike MIN-Gibbs, because of the lack of energy-caching there is no obvious reversibility argument to be made here, and so we can not prove bounds on the spectral gap, since those bounds require reversibility.

\begin{figure*}[t]
  \centering
  \subfigure[]
  {
    \includegraphics[width=0.30\linewidth]{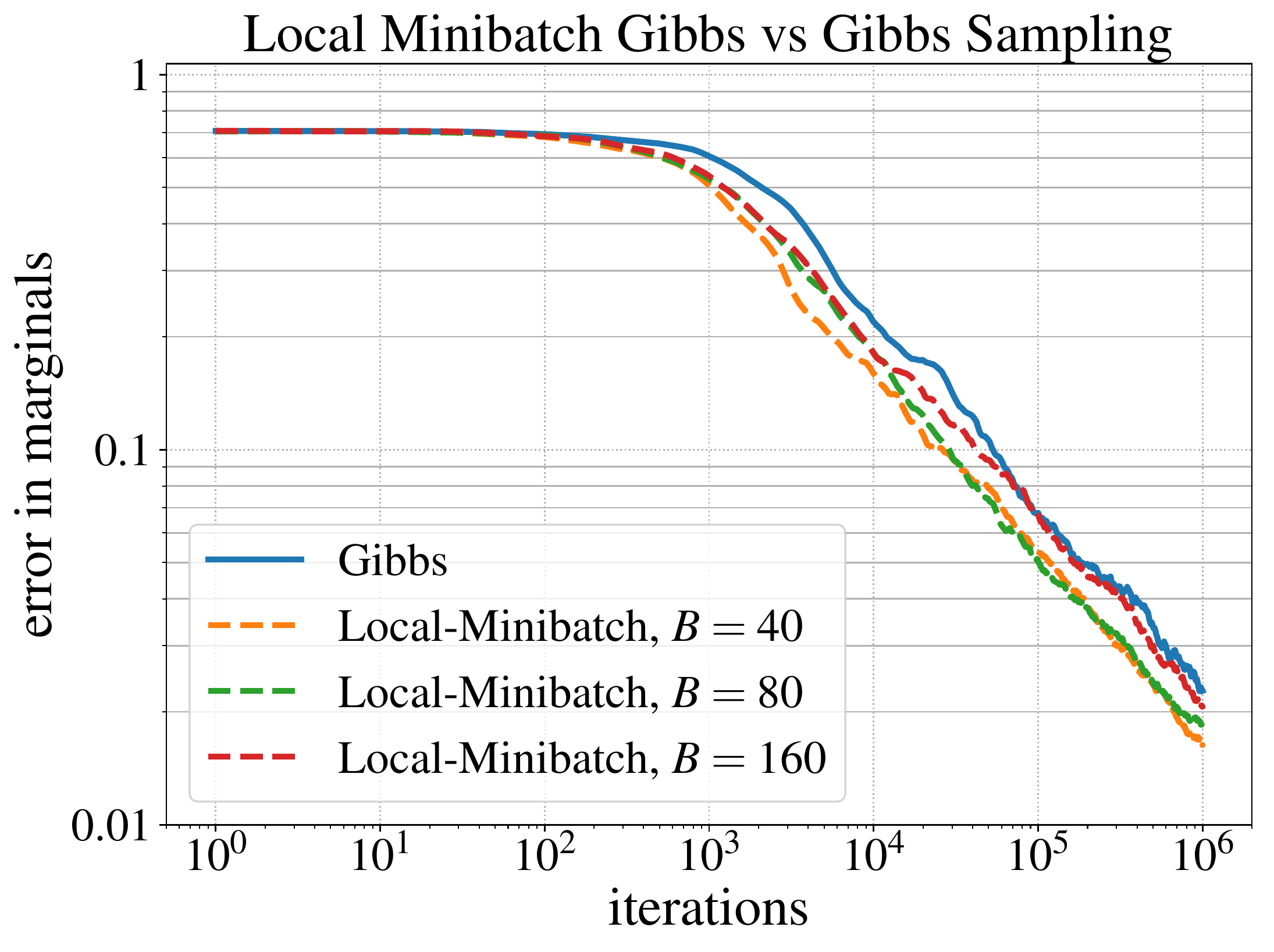}
    \label{fig:local}
  }
  \subfigure[]
  {
    \includegraphics[width=0.30\linewidth]{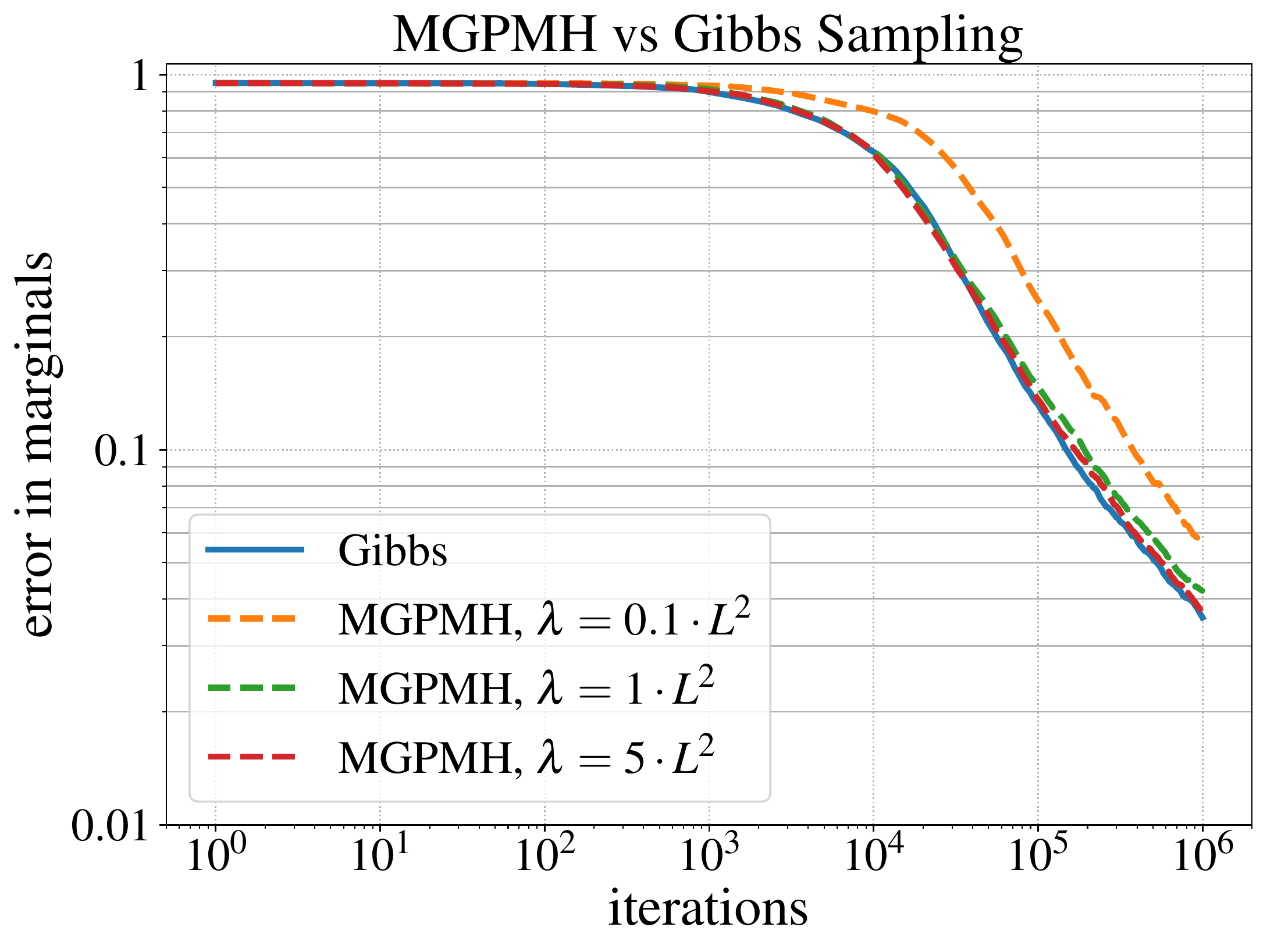}
    \label{fig:MGPMH}
  }
  \subfigure[]
  {
    \includegraphics[width=0.30\linewidth]{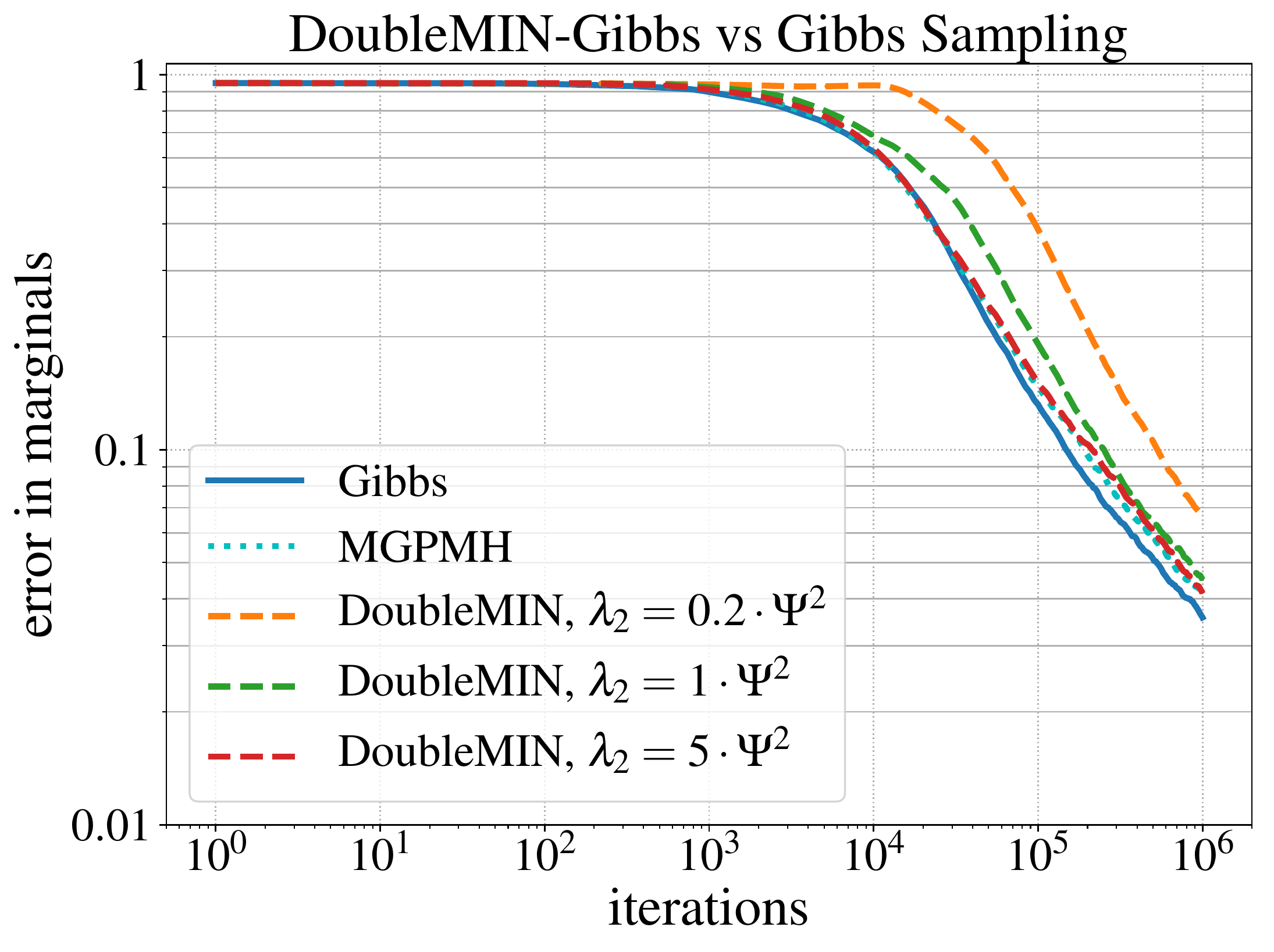}
    \label{fig:double}
  }
  \vspace{-1em}
  \caption{Convergence of (a) Local Minibatch Gibbs, (b) MGPMH, and (c) DoubleMIN-Gibbs, compared with vanilla Gibbs sampling.}
  \label{fig:convergence}
\end{figure*}

\section{Minibatch-Gibbs-Proposal MH}

We left off in the previous section by presenting Algorithm~\ref{algLocalMinibatchGibbs}, which we showed has promising computational cost but gives us no guarantees on accuracy or convergence rate.
There is a general technique that we can use to transform such chains into ones that \emph{do} have accuracy guarantees: Metropolis-Hastings~\cite{hastings1970monte}.
This well-known technique uses updates from an arbitrary Markov chain, called the \emph{proposal distribution}, but then chooses whether or not to \emph{reject} the update based on the true target distribution $\pi$.
This rejection step reshapes the proposal chain into one that is reversible with stationary distribution $\pi$. 
A natural next step for us is to use Metropolis-Hastings with Algorithm~\ref{algLocalMinibatchGibbs} as a proposal distribution.
Doing this results in Algorithm~\ref{algMGPMH}.\footnote{Note that Algorithm~\ref{algMGPMH} is not \emph{precisely} Metropolis-Hastings, because its acceptance probability differs somewhat from what Metropolis-Hastings would have as it is dependent on prior randomness (the selected variable $i$ and minibatch weights $s_{\phi}$).}

\begin{algorithm}[t]
  \caption{MGPMH: Minibatch-Gibbs-Proposal MH}
  \label{algMGPMH}
    \begin{algorithmic}
    \REQUIRE Initial model $x \in \Omega$ and average batch size $\lambda$
    \LOOP
      \STATE Sample $i$ uniformly from $\{1, \ldots, n \}$.
      \FORALL{$\phi$ \textbf{in} $A[i]$}
        \STATE Sample $s_{\phi} \sim \text{Poisson}\left( \frac{ \lambda M_{\phi} }{L} \right)$
      \ENDFOR
      \STATE $S \leftarrow \{ \phi | s_{\phi} > 0 \}$
      \FORALL{$u$ \textbf{in} $\{1, \ldots, D\}$}
        \STATE $\epsilon_u \leftarrow \sum_{\phi \in S} \frac{ s_{\phi} L }{ \lambda M_{\phi} } \phi(x)$
      \ENDFOR
      \STATE construct distribution $\psi(v) \propto \exp\left( \epsilon_v \right)$
      \STATE \textbf{sample} $v$ from $\psi$.
      \STATE construct update candidate $y \leftarrow x$; $y(i) \leftarrow v$
      \STATE compute the update probability
      \[
        a
        =
        \frac{
          \exp\left(
            \sum_{\phi \in A[i]} \phi(y)
          \right)
        }{
          \exp\left(
            \sum_{\phi \in A[i]} \phi(x)
          \right)
        }
        \cdot
        \frac{
          \exp\left(
            \epsilon_{x(i)}
          \right)
        }{
          \exp\left(
            \epsilon_{y(i)}
          \right)
        }.
      \]
      \STATE Update $x \leftarrow y$ with probability $\min(a, 1)$.
    \ENDLOOP
  \end{algorithmic}
  \end{algorithm}

Because Algorithm~\ref{algMGPMH} is based on Metropolis-Hastings, it is natural for it to be reversible and have $\pi$ as its stationary distribution.
We prove that this must be the case.
We also prove a bound on the spectral gap of the chain. 
\begin{theorem}
  \label{thmMGPMHreversible}
  Algorithm~\ref{algMGPMH} is reversible, and it has stationary distribution $\pi$.
\end{theorem}

\begin{theorem}
  \label{thmMGPMHgap}
  Let $\bar \gamma$ be the spectral gap of MGPMH, and let $\gamma$ be the spectral gap of a vanilla Gibbs sampling chain running on the same factor graph.
  Suppose the expected minibatch size is large enough that $L \le \lambda$.
  Then,
  \[
    \bar \gamma
    \ge
    \exp\left(
      -L^2 / \lambda
    \right)
    \cdot
    \gamma.
  \]
\end{theorem}

These theorems mean that MGPMH will converge to the correct stationary distribution $\pi$, and will do so with a convergence rate that is at most a factor of $\exp(L^2/\lambda)$ slower than the vanilla Gibbs chain.
By making $\lambda \approx L^2$, this difference in convergence rates can be made $O(1)$.

On the other hand, when we look at the computational cost of an iteration of Algorithm~\ref{algMGPMH}, we notice it will be $O(\Delta + D \Abs{S})$, since we spend $O(\Delta)$ time computing the minibatch coefficients $s_{\phi}$ and the acceptance probability $a$, and we spend $O(D \Abs{S})$ time computing the energy estimates $\epsilon_u$.
In expectation, the minibatch size $\Abs{S}$ will be
\[
  \textstyle
  \Exv{\Abs{S}} \le \sum_\phi s_{\phi} = \sum_\phi \frac{\lambda M_{\phi}}{L} = \lambda,
\]
so our average runtime for an iteration will be $O(\lambda D + \Delta)$.
If we want a guaranteed $O(1)$-factor-difference on the spectral gap, we need to set $\lambda = O(L^2)$.
Doing this produces a final computation cost of $O(L^2 D + \Delta)$ for MGPMH.
This can represent a significant speedup over the $O(D \Delta)$ complexity of vanilla Gibbs sampling.

\paragraph{Validation of MGPMH.} We again turn to a synthetic example to validate the theoretical guarantees in Algorithm~\ref{algMGPMH}. We simulate a generalization of the Ising model known as the Potts model ~\cite{potts1952some} with domain $\{1, \ldots, D\}$. The energy of a configuration is the following:
\[
  \textstyle
  \zeta_{\text{Potts}} = \sum_{i = 1}^N \sum_{j = 1}^N \beta \cdot A_{ij} \cdot \delta(x(i), x(j))
\]
where the the $\delta$ function equals one whenever $x(i) = x(j)$ and zero otherwise. 
We simulate a graph with $n = N^2 = \Delta + 1 = 400$, $\beta = 4.6$, $D = 10$, and a fully connected configuration $A_{ij}$ that depends on site distance in the same way as before.
\cmd{This model has $L = 5.09$ and $\Psi = 957.1$: note that $L^2 \ll \Delta$ for this model.}
We run the simulation for one million iterations and illustrate the error in the marginals of MGPMH and that of vanilla Gibbs sampling in Figure~\ref{fig:MGPMH}. 
We evaluate MGPMH for three average batch sizes $\lambda$, written in Figure~\ref{fig:MGPMH} as multiples of the local maximum energy squared ($L^2$).  MGPMH approaches vanilla Gibbs sampling as batch size increases, which validates Theorem~\ref{thmMGPMHgap}.

\textbf{Combining methods. } Still, under some conditions, even MGPMH might be too slow.
For example, $\Delta$ could be very large relative to $L^2 D$.
In that setting, even though MGPMH would be faster than Gibbs sampling because it decouples the dependence on $D$ and $\Delta$, it still might be intractably slow.
Can the decoupling of $D$ and $\Delta$ from MGPMH be combined with the $\Delta$-independence of MIN-Gibbs?

The most straightforward way to address this question is to replace the exact energy computation in the acceptance probability of MGPMH with a \emph{second} minibatch approximation, like in MIN-Gibbs.
Doing this combines the effects of MIN-Gibbs, which conceptually replaces $\Delta$ with $\Psi^2$ in the computational cost, and MGPMH, which conceptually replaces $D \Delta$ with $D L^2 + \Delta$.
We call this algorithm \emph{DoubleMIN-Gibbs} because of its doubly minibatched nature.
It turns out DoubleMIN-Gibbs, Algorithm~\ref{algDoubleMINTGibbs}, has the same stationary distribution as MIN-Gibbs, and we can also prove a similar bound on its spectral gap.

\begin{algorithm}[t]
  \caption{DoubleMIN-Gibbs: Doubly-Minibatched}
  \label{algDoubleMINTGibbs}
    \begin{algorithmic}
    \REQUIRE Initial state $(x, \xi_x) \in \Omega \times R$
    \REQUIRE Average first batch size $\lambda_1$
    \REQUIRE Second minibatch estimators $\mu_x$ for $x \in \Omega$
    \LOOP
      \STATE Sample $i$ uniformly from $\{1, \ldots, n \}$.
      \FORALL{$\phi$ \textbf{in} $A[i]$}
        \STATE Sample $s_{\phi} \sim \text{Poisson}\left( \frac{ \lambda M_{\phi} }{L} \right)$
      \ENDFOR
      \STATE $S \leftarrow \{ \phi | s_{\phi} > 0 \}$
      \FORALL{$u$ \textbf{in} $\{1, \ldots, D\}$}
        \STATE $\epsilon_u \leftarrow \sum_{\phi \in S} \frac{ s_{\phi} L }{ \lambda M_{\phi} } \phi(x)$
      \ENDFOR
      \STATE construct distribution $\psi(v) \propto \exp\left( \epsilon_v \right)$
      \STATE \textbf{sample} $v$ from $\psi$.
      \STATE construct update candidate $y \leftarrow x$; $y(i) \leftarrow v$
      \STATE \textbf{sample} $\xi_y \sim \mu_y$
      \STATE compute the update probability
      \[
        a
        =
        \exp\left(\xi_y - \xi_x + \epsilon_{x(i)} - \epsilon_{y(i)}\right).
      \]
      \STATE Update $(x, \xi_x) \leftarrow (y, \xi_y)$ with probability $\min(a, 1)$.
    \ENDLOOP
  \end{algorithmic}
  \end{algorithm}

\begin{theorem}
  \label{thmDoubleMINTreversible}
  The Markov chain described in Algorithm~\ref{algDoubleMINTGibbs} is reversible and has the same stationary distribution (and therefore the same marginal stationary distribution) as MIN-Gibbs with the same estimator (as described in Theorem~\ref{thmMINTreversible}).
\end{theorem}

\begin{theorem}
  \label{thmDoubleMINTgap}
  Let $\bar \gamma$ be the spectral gap of DoubleMIN-Gibbs running with an energy estimator $\mu_x$ that has finite support and that satisfies, for some constant $\delta > 0$ and every $x \in \Omega$,
  \[
    \Prob[\xi_x \sim \mu_x]{\Abs{ \xi_x - \zeta(x) } \le \delta} = 1.
  \]
  Let $\gamma$ be the spectral gap of a MGPMH sampling chain running using the same graph and average batch size.
  Then,
  \[
    \bar \gamma
    \ge
    \exp(-4 \delta) \cdot \gamma.
  \]
\end{theorem}

Essentially, this theorem tells us the same thing Theorem~\ref{thmMINTgap} told us about MIN-Gibbs.
As long as we can bound our estimator to be at most $O(1)$ away from the true energy $\zeta(x)$, convergence is slowed down by only at most a constant factor from MGPMH.
By Lemma~\ref{lemmaUnbiasedConcentration}, this happens with high probability when we use an estimator of the same average batch size (for the second minibatch) as we used for MIN-Gibbs: $B = O(\Psi^2)$.
Of course, this will only ensure an $O(1)$ difference from MGPMH; to ensure an $O(1)$ difference from vanilla Gibbs, we also need to invoke Theorem~\ref{thmMGPMHgap}, which tells us we need to use an average batch size of $O(L^2)$ for the first minibatch.
Along with the fact that the first minibatch sum needs to be computed $D$ times, this results in an overall computational complexity of $O(L^2 D + \Psi^2)$.

One potential obstacle to an implementation actually achieving this asymptotic rate is the sampling of the Poisson random variables to select $s_{\phi}$.
Naively, since there are up to $\Delta$ potential values of $\phi$, this could take up to $O(\Delta)$ time to compute.
However, it is known to be possible to sample a (sparse) vector of Poisson random variables in time proportional to \emph{the sum of their parameters} instead of the length of the vector.
To illustrate, suppose we want to sample $x_1, \ldots, x_m$ independently where $x_i \sim \textrm{Poisson}(\lambda_i)$.
To do this fast, first we notice if we let $B = \sum_{i=1}^m x_i$, then $B$ will \emph{also} be Poisson distributed, with parameter $\Lambda = \sum_{i=1}^m \lambda_i$.
Conditioned on their sum $B$, the variables $x_1, \ldots, x_m$ have a \emph{multinomial distribution} with trial count $B$ and event probabilities $p_i = \lambda_i / \Lambda$.
It is straightforward to sample from a multinomial distribution in time proportional to its trial count (ignoring log factors).
Thus, we can sample $x_1, \ldots, x_m$ by first sampling $B \sim \textrm{Poisson}(\lambda)$ and then sampling $(x_1, \ldots, x_m) \sim \textrm{Multinomial}(B, (p_1, \ldots, p_m))$, and this entire process will only take on average $O(\Lambda)$ time.\footnote{This still requires $O(m)$ time to compute $\Lambda$ and the probabilities $p_i$, but this can be computed once at the start of Algorithm~\ref{algDoubleMINTGibbs}, thereby not affecting the per-iteration computational complexity.}
For Algorithm~\ref{algDoubleMINTGibbs}, this means we can sample all the $s_{\phi}$ in average time $O(\lambda)$.\footnote{This technique can also be used to speed up the other algorithms in this paper, yet it is not necessary to establish their asymptotic computational complexity.}
This is enough to confirm that its overall computational complexity is in fact $O(L^2 D + \Psi^2)$.

\textbf{Validation of DoubleMIN-Gibbs.} We evaluated the DoubleMIN-Gibbs algorithm on the same synthetic Potts model that we used for MGPMH.
Figure~\ref{fig:double} illustrates the performance of DoubleMIN-Gibbs with a batch size of $L^2$ for the first (MGPMH) minibatch, while the batch size of the second (MIN-Gibbs) minibatch, which we denote $\lambda_2$ in Figure~\ref{fig:double}, is adjusted to multiples of $\Psi^2$. As the second minibatch size increases, DoubleMIN Gibbs approaches the trajectory of MGPMH and vanilla Gibbs sampling, which is what we would expect from the result of Theorem~\ref{thmDoubleMINTgap}.

\section{Conclusion}

We studied applying minibatching to Gibbs sampling.
First, we introduced MIN-Gibbs, which improves the asymptotic per-iteration computational cost of Gibbs sampling from $O(D \Delta)$ to $O(D \Psi^2)$ in the setting where the total maximum energy $\Psi^2$ is small compared to the maximum degree $\Delta$.
Second, we introduced MGPMH, which has an asymptotic cost of $O(D L^2 + \Delta)$ and is an improvement in the setting where $L^2 \ll \Delta$.
Finally, we combined the two techniques to produce DoubleMIN-Gibbs, which achieves a computational cost of $O(D L^2 + \Psi^2)$ and further improves over the other algorithms when $L^2 \le \Psi^2 \ll \Delta$.
We proved that all these algorithms can be made to be unbiased, and that these computation costs can be achieved with parameter settings that are guaranteed to have the same asymptotic convergence rate as plain Gibbs sampling, up to a constant-factor slowdown that can be made arbitrarily small.
Our techniques will potentially enable graphical model inference engines that use Gibbs sampling to scale up to much larger and more complicated graphs than were previously possible, and could become a part of a future of highly scalable probabilistic inference on big data.

\section*{Acknowledgments}

The authors acknowledge the support of NSF DMS-1721550 and 1407557.

The views and conclusions contained herein are those of
the authors and should not be interpreted as necessarily rep-
resenting the official policies or endorsements,  either ex-
pressed or implied, of the NSF or the U.S. Government.

\bibliographystyle{plainnat}
\bibliography{references}

\begin{thebibliography}{21}
\providecommand{\natexlab}[1]{#1}
\providecommand{\url}[1]{\texttt{#1}}
\expandafter\ifx\csname urlstyle\endcsname\relax
  \providecommand{\doi}[1]{doi: #1}\else
  \providecommand{\doi}{doi: \begingroup \urlstyle{rm}\Url}\fi

\bibitem[Bardenet et~al.(2014)Bardenet, Doucet, and
  Holmes]{bardenet2014towards}
R{\'e}mi Bardenet, Arnaud Doucet, and Chris Holmes.
\newblock Towards scaling up {Markov} chain {Monte Carlo}: an adaptive
  subsampling approach.
\newblock In \emph{International Conference on Machine Learning}, pages
  405--413, 2014.

\bibitem[De~Sa et~al.(2016)De~Sa, R{\'e}, and Olukotun]{desa2016hoggibbs}
Christopher De~Sa, Christopher R{\'e}, and Kunle Olukotun.
\newblock Ensuring rapid mixing and low bias for asynchronous {Gibbs} sampling.
\newblock \emph{International Conference on Machine Learning (ICML)}, pages
  1567--1576, 2016.

\bibitem[Hastings(1970)]{hastings1970monte}
W~Keith Hastings.
\newblock {Monte Carlo} sampling methods using {Markov} chains and their
  applications.
\newblock \emph{Biometrika}, 57\penalty0 (1):\penalty0 97--109, 1970.

\bibitem[He et~al.(2016)He, De~Sa, Mitliagkas, and R{\'e}]{he2016scan}
Bryan He, Christopher De~Sa, Ioannis Mitliagkas, and Christopher R{\'e}.
\newblock Scan order in {Gibbs} sampling: Models in which it matters and bounds
  on how much.
\newblock \emph{Advances in neural information processing systems (NIPS)},
  2016.

\bibitem[Ising(1925)]{ising1925beitrag}
Ernst Ising.
\newblock Beitrag zur theorie des ferromagnetismus.
\newblock \emph{Zeitschrift f{\"u}r Physik A Hadrons and Nuclei}, 31\penalty0
  (1):\penalty0 253--258, 1925.

\bibitem[Johndrow et~al.(2015)Johndrow, Mattingly, Mukherjee, and
  Dunson]{johndrow2015approximations}
James~E Johndrow, Jonathan~C Mattingly, Sayan Mukherjee, and David Dunson.
\newblock Approximations of {Markov} chains and {Bayesian} inference.
\newblock \emph{arXiv preprint arXiv:1508.03387}, 2015.

\bibitem[Koller and Friedman(2009)]{koller2009probabilistic}
Daphne Koller and Nir Friedman.
\newblock \emph{Probabilistic graphical models: principles and techniques}.
\newblock MIT press, 2009.

\bibitem[Korattikara et~al.(2014)Korattikara, Chen, and
  Welling]{korattikara2014austerity}
Anoop Korattikara, Yutian Chen, and Max Welling.
\newblock Austerity in {MCMC} land: Cutting the {Metropolis-Hastings} budget.
\newblock In \emph{International Conference on Machine Learning (ICML)}, pages
  181--189, 2014.

\bibitem[Levin et~al.(2009)Levin, Peres, and Wilmer]{levin2009markov}
David~Asher Levin, Yuval Peres, and Elizabeth~Lee Wilmer.
\newblock \emph{Markov chains and mixing times}.
\newblock American Mathematical Soc., 2009.

\bibitem[Li and Wong(2017)]{li2017mini}
Dangna Li and Wing~H Wong.
\newblock Mini-batch tempered {MCMC}.
\newblock \emph{arXiv preprint arXiv:1707.09705}, 2017.

\bibitem[Lunn et~al.(2009)Lunn, Spiegelhalter, Thomas, and Best]{lunn2009bugs}
David Lunn, David Spiegelhalter, Andrew Thomas, and Nicky Best.
\newblock The {BUGS} project: evolution, critique and future directions.
\newblock \emph{Statistics in medicine}, \penalty0 (25):\penalty0 3049--3067,
  2009.

\bibitem[Maclaurin and Adams(2014)]{maclaurin2014firefly}
Dougal Maclaurin and Ryan~P Adams.
\newblock Firefly {Monte Carlo}: exact {MCMC} with subsets of data.
\newblock In \emph{Proceedings of the Thirtieth Conference on Uncertainty in
  Artificial Intelligence}, pages 543--552. AUAI Press, 2014.

\bibitem[McCallum et~al.(2009)McCallum, Schultz, and
  Singh]{mccallum2009factorie}
Andrew McCallum, Karl Schultz, and Sameer Singh.
\newblock Factorie: Probabilistic programming via imperatively defined factor
  graphs.
\newblock In \emph{NIPS}, pages 1249--1257, 2009.

\bibitem[Newman et~al.(2007)Newman, Smyth, Welling, and
  Asuncion]{newman2007distributed}
David Newman, Padhraic Smyth, Max Welling, and Arthur~U Asuncion.
\newblock Distributed inference for latent dirichlet allocation.
\newblock In \emph{Advances in neural information processing systems (NIPS)},
  pages 1081--1088, 2007.

\bibitem[Potts(1952)]{potts1952some}
Renfrey~Burnard Potts.
\newblock Some generalized order-disorder transformations.
\newblock In \emph{Mathematical proceedings of the cambridge philosophical
  society}, volume~48, pages 106--109. Cambridge University Press, 1952.

\bibitem[{Quiroz} et~al.(2018){Quiroz}, {Tran}, {Villani}, {Kohn}, and
  {Dang}]{quiroz2018mcmc}
M.~{Quiroz}, M.-N. {Tran}, M.~{Villani}, R.~{Kohn}, and K.-D. {Dang}.
\newblock {The block-Poisson estimator for optimally tuned exact subsampling
  MCMC}.
\newblock \emph{ArXiv e-prints}, March 2018.

\bibitem[Rekatsinas et~al.(2017)Rekatsinas, Chu, Ilyas, and
  R{\'e}]{rekatsinas2017holoclean}
Theodoros Rekatsinas, Xu~Chu, Ihab~F Ilyas, and Christopher R{\'e}.
\newblock Holoclean: Holistic data repairs with probabilistic inference.
\newblock \emph{Proceedings of the VLDB Endowment}, 10\penalty0 (11):\penalty0
  1190--1201, 2017.

\bibitem[Seita et~al.(2017)Seita, Pan, Chen, and Canny]{chen2016efficient}
Daniel Seita, Xinlei Pan, Haoyu Chen, and John~F. Canny.
\newblock An efficient minibatch acceptance test for metropolis-hastings.
\newblock \emph{Conference on Uncertainty in Artificial Intelligence (UAI)},
  abs/1610.06848, 2017.

\bibitem[Smola and Narayanamurthy(2010)]{smola2010architecture}
Alexander Smola and Shravan Narayanamurthy.
\newblock An architecture for parallel topic models.
\newblock \emph{Proceedings of the VLDB Endowment (PVLDB)}, 2010.

\bibitem[Theis et~al.(2012)Theis, Sohl-dickstein, and Bethge]{NIPS2012_4832}
Lucas Theis, Jascha Sohl-dickstein, and Matthias Bethge.
\newblock Training sparse natural image models with a fast {Gibbs} sampler of
  an extended state space.
\newblock In \emph{Advances in neural information processing systems (NIPS)},
  pages 1124--1132. 2012.

\bibitem[Zhang and R{\'e}(2014)]{zhang2014dimmwitted}
Ce~Zhang and Christopher R{\'e}.
\newblock {DimmWitted}: A study of main-memory statistical analytics.
\newblock \emph{PVLDB}, 2014.

\end{thebibliography}

\newpage
\onecolumn

\appendix
\section{Proofs}

\begin{proof}[Detailed proof of Lemma~\ref{lemmaEstimatorUnbiased}]
  The expected value of the exponential of the estimator defined in (\ref{eqnUnbiasedEst}) is
  \begin{align*}
    \Exv{\exp(\epsilon_x)}
    =
    \Exv{\exp\left( \sum_{\phi \in S} s_{\phi} \log\left( 1 + \frac{\Psi}{\lambda M_{\phi}} \phi(x) \right) \right)}.
  \end{align*}
  Since the $s_{\phi}$ are independent, this becomes
  \begin{align*}
    \Exv{\exp(\epsilon_x)}
    =
    \prod_{\phi \in \Phi} \Exv{\exp\left( s_{\phi} \log\left( 1 + \frac{\Psi}{\lambda M_{\phi}} \phi(x) \right) \right)}.
  \end{align*}
  Each of these constituent expected values is an evaluation of the moment generating function of a Poisson random variable. Applying the known expression for this MGF,
  \begin{align*}
    \Exv{\exp(\epsilon_x)}
    =&
    \prod_{\phi \in \Phi} 
    \exp\left(
      \frac{ \lambda M_{\phi} }{ \Psi }
      \left(
        \exp\left(
          \log\left( 1 + \frac{\Psi}{\lambda M_{\phi}} \phi(x) \right)
        \right)
        -
        1
      \right)
    \right) \\
    =&
    \prod_{\phi \in \Phi} 
    \exp\left(
      \frac{ \lambda M_{\phi} }{ \Psi }
      \left(
        \frac{\Psi}{\lambda M_{\phi}} \phi(x)
      \right)
    \right) \\
    =&
    \prod_{\phi \in \Phi} 
    \exp\left(
      \phi(x)
    \right) \\
    =&
    \exp(\zeta(x)).
  \end{align*}
  This is what we wanted to show.
\end{proof}

\begin{proof}[Proof of Theorem~\ref{thmMINTreversible}]
  The probability of transitioning from $(x, \epsilon_{x(i)})$ to $(y, \epsilon_{y(i)})$ for two distinct states $x$ and $y$ which differ only in variable $i$ will be the probability that: we decide to re-sample variable $i$, we choose this particular value $\epsilon_{y(i)}$ when we sample it at random from $\mu_y$, and we sample $v = y(i)$ from $\rho$ when we sample the new value of variable $i$.
  So the transition probability will be
  \[
    T((x, \epsilon_{x(i)}), (y, \epsilon_{y(i)}))
    =
    \frac{1}{n} \cdot \mu_y(\epsilon_{y(i)}) \cdot \Exv{ \frac{\exp(\epsilon_{y(i)})}{\sum_{w=1}^D \exp(\epsilon_w)} }
  \]
  where this expected value is taken over the randomly sampled $\epsilon_w$ (except when $w \in \{ x(i), y(i) \}$ where the value of $\epsilon_w$ is already determined).
  If we multiply both sides by $\bar \pi(x, \epsilon)$ which by definition for some fixed $Z$ is equal to
  \[
    \bar \pi(x, \epsilon) = \frac{1}{Z} \mu_x(\epsilon) \cdot \exp(\epsilon)
  \]
  then we get
  \begin{align*}
    &\bar \pi(x, \epsilon) T((x, \epsilon_{x(i)}), (y, \epsilon_{y(i)})) \\
    &\hspace{2em}=
    \frac{1}{n Z} \cdot \Exv{ \frac{\mu_x(\epsilon) \cdot \exp(\epsilon) \cdot \mu_y(\epsilon_{y(i)}) \cdot \exp(\epsilon_{y(i)}) }{\sum_{w=1}^D \exp(\epsilon_w)} }
  \end{align*}
  and this expression is clearly symmetric in $x$ and $y$, so the chain is reversible and $\bar \pi$ is indeed its stationary distribution.
  The second result, about the marginal distribution in $x$, follows directly from the definition of expected value.
\end{proof}

\begin{proof}[Proof of Theorem~\ref{thmMINTgap}]
  This proof will use the technique of \emph{Dirichlet forms}~\cite{levin2009markov}.
  The Dirichlet form of a Markov chain with transition matrix $T$ is defined for function argument $f: \Omega \rightarrow \R$ as
  \[
    \mathcal{E}(f) = \frac{1}{2} \sum_{x,y \in \Omega} (f(x) - f(y))^2 \pi(x) T(x, y).
  \]
  Similarly, the variance of the function $f$ under the distribution $\pi$ is defined as
  \[
    \Var[\pi]{f} = \frac{1}{2} \sum_{x,y \in \Omega} (f(x) - f(y))^2 \pi(x) \pi(y);
  \]
  this is equivalent to the standard definition of variance.
  It is a standard result (from \citet{levin2009markov}) that the spectral gap can be written as
  \[
    \gamma = \min_f \frac{\mathcal{E}(f)}{\Var[\pi]{f}}.
  \]
  Now, consider the Dirichlet form of the MIN-Gibbs chain.
  Let this form be $\bar{\mathcal{E}}(f)$, where for $f: \Omega \times \R \rightarrow \R$,
  \begin{dmath*}
    \bar{\mathcal{E}}(f)
    =
    \frac{1}{2} \sum_{x,y \in \Omega}
    \sum_{\epsilon_x \in \dom(\mu_x)}
    \sum_{\epsilon_y \in \dom(\mu_y)} 
    (f(x, \epsilon_x) - f(y, \epsilon_y))^2 \bar \pi(x, \epsilon_x) \bar T((x, \epsilon_x), (y, \epsilon_y)).
  \end{dmath*}
  From the result of Theorem~\ref{thmMINTreversible}, we know that for some $\bar Z$,
  \[
    \bar \pi(x, \epsilon) = \frac{1}{\bar Z} \mu_x(\epsilon) \cdot \exp(\epsilon).
  \]
  We also recall from the proof of that theorem that, for $x$ and $y$ which differ only in variable $i$,
  \[
    \bar T((x, \epsilon_{x}), (y, \epsilon_{y}))
    =
    \frac{1}{n} \cdot \mu_y(\epsilon_{y}) \cdot \Exv{ \frac{\exp(\epsilon_{y})}{\sum_{w=1}^D \exp(\epsilon_w)} },
  \]
  where the $\epsilon_w$ are each sampled independently from $\mu_w$.
  Otherwise, the transition probability is zero.
  As a result, if we let $Q \subset \Omega \times \Omega$ denote the pairs of states which differ only in a single variable, we can rewrite our Dirichlet form as
  \begin{dmath*}
    \bar{\mathcal{E}}(f)
    =
    \frac{1}{2 n \bar Z} \sum_{(x,y) \in Q}
    \sum_{\epsilon_x \in \dom(\mu_x)}
    \sum_{\epsilon_y \in \dom(\mu_y)} 
    (f(x, \epsilon_x) - f(y, \epsilon_y))^2
    \mu_x(\epsilon_x) \cdot \exp(\epsilon_x) \cdot 
    \mu_y(\epsilon_y) \cdot \Exv{ \frac{\exp(\epsilon_y)}{\sum_{w=1}^D \exp(\epsilon_w)} }.
  \end{dmath*}
  Now, by the definition of expected value, if we suppose that $\epsilon_x$ and $\epsilon_y$ are random variables sampled from $\mu_x$ and $\mu_y$ respectively, then
  \begin{dmath*}
    \bar{\mathcal{E}}(f)
    =
    \frac{1}{2 n \bar Z} \sum_{(x,y) \in Q}
    \Exv{
	    (f(x, \epsilon_x) - f(y, \epsilon_y))^2
      \cdot
	    \frac{\exp(\epsilon_x) \exp(\epsilon_y)}{\sum_{w=1}^D \exp(\epsilon_w)}
    }.
  \end{dmath*}
  Using a similar argument, we can also analyze the original vanilla Gibbs chain.
  For simplicity, define
  \[
    \zeta(x) = \sum_{\phi \in \Phi} \phi(x),
  \]
  and suppose that for some $Z$
  \[
    \pi(x) = \frac{1}{Z} \exp(\zeta(x)).
  \]
  Now we can define the Dirichlet form of the original Gibbs chain as $\mathcal{E}(g)$, where $g: \Omega \rightarrow \R$ and
  \begin{dmath*}
    \mathcal{E}(g)
    =
    \frac{1}{2 n Z} \sum_{(x,y) \in Q}
    (g(x) - g(y))^2
    \cdot
    \frac{\exp(\zeta(x)) \exp(\zeta(y))}{\sum_{w=1}^D \exp(\zeta(w))}.
  \end{dmath*}
  Now, we know by the condition of the theorem that
  \[
    \Abs{\epsilon_x - \zeta(x)} \le \delta.
  \]
  As a result, we can bound the Dirichlet form of our MIN-Gibbs chain with
  \begin{dmath*}
    \bar{\mathcal{E}}(f)
    \ge
    \frac{1}{2 n \bar Z} \sum_{(x,y) \in Q}
    \Exv{
      (f(x, \epsilon_x) - f(y, \epsilon_y))^2
      \cdot
      \frac{\exp(\zeta(x) - \delta) \exp(\zeta(y) - \delta)}{\sum_{w=1}^D \exp(\zeta(w) + \delta)}
    }
    =
    \frac{1}{2 n \bar Z} \sum_{(x,y) \in Q}
    \frac{\exp(\zeta(x) - \delta) \exp(\zeta(y) - \delta)}{\sum_{w=1}^D \exp(\zeta(w) + \delta)}
    \Exv{ (f(x, \epsilon_x) - f(y, \epsilon_y))^2 }.
  \end{dmath*}
  On the other hand, the variance form associated with the MIN-Gibbs chain is
  \begin{dmath*}
    \Var[\bar \pi]{f}
    =
    \frac{1}{2} \sum_{x,y \in \Omega} 
    \sum_{\epsilon_x \in \dom(\mu_x)}
    \sum_{\epsilon_y \in \dom(\mu_y)} 
    (f(x, \epsilon_x) - f(y, \epsilon_y))^2 \bar \pi(x) \bar \pi(y)
    =
    \frac{1}{2 \bar Z^2} \sum_{x,y \in \Omega} 
    \sum_{\epsilon_x \in \dom(\mu_x)}
    \sum_{\epsilon_y \in \dom(\mu_y)} 
    (f(x, \epsilon_x) - f(y, \epsilon_y))^2 \mu_x(\epsilon_x) \cdot \exp(\epsilon_x) \cdot \mu_y(\epsilon_y) \cdot \exp(\epsilon_y)
    \le
    \frac{1}{2 \bar Z^2} \sum_{x,y \in \Omega}
    \exp(\zeta(x) + \delta) \cdot \exp(\zeta(y) + \delta) \cdot
    \Exv{ (f(x, \epsilon_x) - f(y, \epsilon_y))^2 }.
  \end{dmath*}
  Now, we need some way to get rid of these expected values.
  One way to do it is to recall that for positive numbers $a_1, a_2, \ldots, a_N$ and $b_1, b_2, \ldots, b_N$,
  \[
    \frac{\sum_{i=1}^n a_i}{\sum_{i=1}^n b_i} \ge \min_i \frac{a_i}{b_i}
  \]
  or, in terms of expected value, if $a$ and be are nonnegative functions and $Z$ is a random variable,
  \[
    \frac{\Exv{a(Z)}}{\Exv{b(Z)}}
    =
    \frac{1}{\Exv{b(Z)}} \Exv{b(Z) \frac{a(Z)}{b(Z)}}
    \ge
    \frac{1}{\Exv{b(Z)}} \Exv{b(Z) \min_z \frac{a(z)}{b(z)}}
    =
    \min_z \frac{a(z)}{b(z)}.
  \]
  Equivalently, we can say that there exists a fixed $z$ (the $z$ that minimizes the expression on the right) such that
  \begin{equation}
    \label{eqnABZ}
    \frac{\Exv{a(Z)}}{\Exv{b(Z)}} \ge \frac{a(z)}{b(z)}.
  \end{equation}
  It follows that
  \begin{dmath*}
    \bar \gamma
    =
    \min_f \frac{\bar{\mathcal{E}}(f)}{\Var[\bar \pi]{f}}
    \ge
    \min_f \frac{
      \frac{1}{2 n \bar Z} \sum_{(x,y) \in Q}
      \frac{\exp(\zeta(x) - \delta) \exp(\zeta(y) - \delta)}{\sum_{w=1}^D \exp(\zeta(w) + \delta)}
      \Exv{ (f(x, \epsilon_x) - f(y, \epsilon_y))^2 }
    }{
      \frac{1}{2 \bar Z^2} \sum_{x,y \in \Omega}
      \exp(\zeta(x) + \delta) \cdot \exp(\zeta(y) + \delta) \cdot
      \Exv{ (f(x, \epsilon_x) - f(y, \epsilon_y))^2 }
    }
    =
    \min_f \frac{
      \Exv{ 
        \frac{1}{2 n \bar Z} \sum_{(x,y) \in Q}
        \frac{\exp(\zeta(x) - \delta) \exp(\zeta(y) - \delta)}{\sum_{w=1}^D \exp(\zeta(w) + \delta)}
        (f(x, \epsilon_x) - f(y, \epsilon_y))^2
      }
    }{
      \Exv{ 
        \frac{1}{2 \bar Z^2} \sum_{x,y \in \Omega}
        \exp(\zeta(x) + \delta) \cdot \exp(\zeta(y) + \delta) \cdot
        (f(x, \epsilon_x) - f(y, \epsilon_y))^2
      }
    },
  \end{dmath*}
  where the randomness in this expression is over the random variables $\epsilon_x \sim \mu_x$ for each $x \in \Omega$.
  We can think of this as a ratio of huge sums over all possible assignments of $\epsilon_x$ for all $x$,
  which will be lower bounded by the ratio for some fixed assignment of $\epsilon_x$, as per (\ref{eqnABZ}).
  That is, for some assignment of values to each $\epsilon_x$ for $x \in \Omega$,
  \begin{dmath*}
    \bar \gamma
    =
    \min_f \frac{\bar{\mathcal{E}}(f)}{\Var[\bar \pi]{f}}
    \ge
    \min_f \frac{
      \frac{1}{2 n \bar Z} \sum_{(x,y) \in Q}
      \frac{\exp(\zeta(x) - \delta) \exp(\zeta(y) - \delta)}{\sum_{w=1}^D \exp(\zeta(w) + \delta)}
      (f(x, \epsilon_x) - f(y, \epsilon_y))^2
    }{
      \frac{1}{2 \bar Z^2} \sum_{x,y \in \Omega}
      \exp(\zeta(x) + \delta) \cdot \exp(\zeta(y) + \delta) \cdot
      (f(x, \epsilon_x) - f(y, \epsilon_y))^2
    }.
  \end{dmath*}
  If we now define
  \[
    g(x) = f(x, \epsilon_x)
  \]
  then we can simplify this ratio to
  \begin{dmath*}
    \bar \gamma
    \ge
    \min_g \frac{
      \frac{1}{2 n \bar Z} \sum_{(x,y) \in Q}
      \frac{\exp(\zeta(x) - \delta) \exp(\zeta(y) - \delta)}{\sum_{w=1}^D \exp(\zeta(w) + \delta)}
      (g(x) - g(y))^2
    }{
      \frac{1}{2 \bar Z^2} \sum_{x,y \in \Omega}
      \exp(\zeta(x) + \delta) \cdot \exp(\zeta(y) + \delta) \cdot
      (g(x) - g(y))^2
    }
    =
    \exp(-5 \delta)
    \min_g \frac{
      \frac{1}{2 n} \sum_{(x,y) \in Q}
      \frac{\exp(\zeta(x)) \exp(\zeta(y))}{\sum_{w=1}^D \exp(\zeta(w))}
      (g(x) - g(y))^2
    }{
      \frac{1}{2 \bar Z} \sum_{x,y \in \Omega}
      \exp(\zeta(x)) \cdot \exp(\zeta(y)) \cdot
      (g(x) - g(y))^2
    }.
  \end{dmath*}
  Now, we notice that these are the same as the expressions for the original chain!
  In fact,
  \begin{dmath*}
    \bar \gamma
    \ge
    \exp(-5 \delta)
    \frac{\bar Z}{Z}
    \min_g \frac{
      \mathcal{E}(g)
    }{
      \Var[\pi]{g}
    }.
  \end{dmath*}
  All that remains is to bound this ratio of the $Z$ and $\bar Z$. We can do this with
  \begin{dmath*}
    \frac{\bar Z}{Z}
    =
    \frac{
      \sum_{x \in \Omega} \sum_{\epsilon \in \dom(\mu_x)}
      \mu_x(\epsilon) \exp(\epsilon)
    }{
      \sum_{x \in \Omega} \exp(\zeta(x))
    }
    \ge
    \frac{
      \sum_{x \in \Omega} \sum_{\epsilon \in \dom(\mu_x)}
      \mu_x(\epsilon) \exp(\zeta(x) - \delta)
    }{
      \sum_{x \in \Omega} \exp(\zeta(x))
    }
    =
    \exp(-\delta)
    \frac{
      \sum_{x \in \Omega} \exp(\zeta(x))
    }{
      \sum_{x \in \Omega} \exp(\zeta(x))
    }
    =
    \exp(-\delta),
  \end{dmath*}
  and so
  \begin{dmath*}
    \bar \gamma
    \ge
    \exp(-6 \delta)
    \min_g \frac{
      \mathcal{E}(g)
    }{
      \Var[\pi]{g}
    }
    =
    \exp(-5 \delta) \gamma.
  \end{dmath*}
  This is what we wanted to show.
\end{proof}

\begin{proof}[Proof of Lemma~\ref{lemmaUnbiasedConcentration}]
  The estimator we want to bound is
  \[
    \epsilon_x = \sum_{\phi \in S} s_{\phi} \log\left( 1 + \frac{\Psi}{\lambda M_{\phi}} \phi(x) \right).
  \]
  where the parameter of the Poisson-distributed random variable $s_{\phi}$ is
  \[
    \Exv{s_{\phi}} = \frac{\lambda M_{\phi}}{\Psi}.
  \]
  First, we note that since $0 \le \phi(x)$ and the logarithm is concave,
  \[
    0
    \le
    \log\left( 1 + \frac{\Psi}{\lambda M_{\phi}} \phi(x) \right)
    \le
    \frac{\Psi}{\lambda M_{\phi}} \phi(x)
    \le
    \frac{\Psi}{\lambda}.
  \]
  Second, note that since the variance of a Poisson random variable is equal to its expected value, the variance of the estimator is
  \begin{dmath*}
    \Var{\epsilon_x}
    =
    \sum_{\phi \in \Phi} \Var{s_{\phi}}
    \left( \log\left( 1 + \frac{\Psi}{\lambda M_{\phi}} \phi(x) \right) \right)^2
    \le
    \frac{\Psi^2}{\lambda^2} \sum_{\phi \in \Phi} \Var{s_{\phi}}
    =
    \frac{\Psi^2}{\lambda^2} \sum_{\phi \in \Phi} \frac{\lambda M_{\phi}}{\Psi}
    =
    \frac{\Psi^2}{\lambda}.
  \end{dmath*}
  Thus, by the Bernstein inequality, if we can write $\epsilon_x$ as a sum, then
  \begin{dmath*}
    \Prob{ \Abs{\epsilon_x - \Exv{\epsilon_x}} \ge t }
    \le
    2 \exp\left(
      -\frac{
        \frac{1}{2} t^2
      }{
        \Var{\epsilon_x}
        +
        \frac{1}{3} C t
      }
    \right)
    =
    2 \exp\left(
      -\frac{
        \frac{1}{2} t^2
      }{
        \frac{\Psi^2}{\lambda}
        +
        \frac{1}{3} C t
      }
    \right)
  \end{dmath*}
  where $C$ is the maximum magnitude of any component of the sum.
  But, since the Poisson distribution is infinitely divisible, we can make arbitrarily many components of the sum, so we can push $C$ arbitrarily close to $0$.
  By a continuity argument we can set it equal to $0$, and so
  \begin{dmath*}
    \Prob{ \Abs{\epsilon_x - \Exv{\epsilon_x}} \ge t }
    \le
    2 \exp\left(
      -\frac{
        \lambda t^2
      }{
        2 \Psi^2
      }
    \right).
  \end{dmath*}
  Next, we need to bound $\Exv{\epsilon_x}$.
  First, note that by Jensen's inequality,
  \[
    \exp(\Exv{\epsilon_x})
    \le
    \Exv{\exp(\epsilon_x)}
    =
    \exp(\zeta(x)),
  \]
  so $\epsilon_x$ is an underestimator of $\zeta(x)$ in expectation.
  We thus need to bound it from below.
  Since the expected value of $s_{\phi}$ is known, we can write the expected value explicitly as
  \begin{dmath*}
    \Exv{\epsilon_x}
    =
    \sum_{\phi \in \Phi} 
    \frac{\lambda M_{\phi}}{\Psi}
    \log\left( 1 + \frac{\Psi}{\lambda M_{\phi}} \phi(x) \right).
  \end{dmath*}
  We know for positive $z$ that $\log(1+z) \ge x - x^2/2$, so
  \begin{dmath*}
    \Exv{\epsilon_x}
    \ge
    \sum_{\phi \in \Phi} 
    \frac{\lambda M_{\phi}}{\Psi}
    \left(
      \frac{\Psi}{\lambda M_{\phi}} \phi(x)
      -
      \left( \frac{\Psi}{\lambda M_{\phi}} \phi(x) \right)^2
    \right)
    \ge
    \sum_{\phi \in \Phi} 
    \frac{\lambda M_{\phi}}{\Psi}
    \left(
      \frac{\Psi}{\lambda M_{\phi}} \phi(x)
      -
      \frac{\Psi^2}{\lambda^2}
    \right)
    =
    \sum_{\phi \in \Phi} 
    \left(
      \phi(x)
      -
      \frac{\Psi M_{\phi}}{\lambda}
    \right)
    =
    \zeta(x)
    -
    \frac{\Psi^2}{\lambda}.
  \end{dmath*}
  Thus,
  \begin{dmath*}
    \Prob{ \Abs{\epsilon_x - \zeta(x) } \ge t + \frac{\Psi^2}{\lambda} }
    \le
    2 \exp\left(
      -\frac{
        \lambda t^2
      }{
        2 \Psi^2
      }
    \right).
  \end{dmath*}
  Now, suppose we want $\Prob{ \Abs{\epsilon_x - \zeta(x) } \ge \delta} \le a$.
  If we assign $t = \delta / 2$, and require that
  \[
    \lambda \ge \frac{2 \Psi^2}{\delta},
  \]
  then we just need
  \begin{dmath*}
    a
    \ge
    2 \exp\left(
      -\frac{
        \lambda \delta^2
      }{
        8 \Psi^2
      }
    \right).
  \end{dmath*}
  So it suffices to set
  \[
    \lambda
    \ge
    \frac{
      8 \Psi^2
    }{
      \delta^2
    }
    \log\left( \frac{2}{a} \right).
  \]
  This proves the lemma.
\end{proof}

\begin{proof}[Proof of Theorem~\ref{thmMGPMHreversible}]
  Statistically, the sampling procedure in Algorithm~\ref{algMGPMH} is equivalent to sampling a variable $i$, and then sampling a Poisson random variable $s_{\phi}$ for each $\phi \in \Phi$.
  Let $T_{i,s}(x, y)$ denote the probability of transitioning from state $x$ to $y$ given that we have already chosen to sample variable $i$ with minibatch coefficients $s$.
  Then, the overall transition matrix will be
  \[
    T(x, y) = \Exv{ T_{i,s}(x, y) }
  \]
  where this expected value is taken with respect to the random variables $i$ and $s$.
  If we look at $T_{i,s}(x,y)$ for two states that differ only at variable $i$, it will be equal to
  \[
    T_{i,s}(x,y) = \rho(y(i)) \cdot \min(a, 1),
  \]
  which is the probability of proposing $y$ times the probability of accepting that proposal.
  (Otherwise, if $x$ and $y$ differ at a variable other than $i$, this transition probability will be $T_{i,s}(x,y) = 0$.)
  We can expand this expression to
  \begin{dmath*}
    T_{i,s}(x,y)
    =
    \frac{\exp(\epsilon_{y(i)})}{\sum_{w=1}^D \exp(\epsilon_w)}
    \cdot
    \min\left(
      \frac{
        \exp\left(
          \sum_{\varphi \in A[i]} \varphi(y)
        \right)
      }{
        \exp\left(
          \sum_{\varphi \in A[i]} \varphi(x)
        \right)
      }
      \cdot
      \frac{
        \exp\left(
          \epsilon_x
        \right)
      }{
        \exp\left(
          \epsilon_y
        \right)
      }, 1 \right)
    =
    \frac{\exp(\epsilon_{y(i)})}{\sum_{w=1}^D \exp(\epsilon_w)}
    \cdot
    \min\left(
      \frac{
        \exp\left(
          \sum_{\varphi \in \Phi} \varphi(y)
        \right)
      }{
        \exp\left(
          \sum_{\varphi \in \Phi} \varphi(x)
        \right)
      }
      \cdot
      \frac{
        \exp\left(
          \epsilon_x
        \right)
      }{
        \exp\left(
          \epsilon_y
        \right)
      }, 1 \right)
    =
    \frac{\exp(\epsilon_{y(i)})}{\sum_{w=1}^D \exp(\epsilon_w)}
    \cdot
    \min\left(
      \frac{
        \exp\left( \zeta(y) \right)
      }{
        \exp\left( \zeta(x) \right)
      }
      \cdot
      \frac{
        \exp\left(
          \epsilon_{x(i)}
        \right)
      }{
        \exp\left(
          \epsilon_{y(i)}
        \right)
      }, 1 \right).
  \end{dmath*}
  Multiplying this by $\pi(x)$,
  \begin{dmath*}
    \pi(x) T_{i,s}(x,y)
    =
    \frac{1}{Z} \exp(\zeta(x))
    \cdot
    \frac{\exp(\epsilon_{y(i)})}{\sum_{w=1}^D \exp(\epsilon_w)}
    \cdot
    \min\left(
      \frac{
        \exp\left( \zeta(y) \right)
      }{
        \exp\left( \zeta(x) \right)
      }
      \cdot
      \frac{
        \exp\left(
          \epsilon_{x(i)}
        \right)
      }{
        \exp\left(
          \epsilon_{y(i)}
        \right)
      }, 1 \right)
    =
    \frac{1}{Z \sum_{w=1}^D \exp(\epsilon_w)}
    \cdot
    \min\left(
      \exp\left( \zeta(y) \right)
      \cdot
      \exp\left( \epsilon_{x(i)} \right)
      , 
      \exp\left( \zeta(x) \right)
      \cdot
      \exp\left( \epsilon_{y(i)} \right)
    \right).
  \end{dmath*}
  This last expression is clearly symmetric in $x$ and $y$. So,
  \[
    \pi(x) T_{i,s}(x,y)
    =
    \pi(y) T_{i,s}(y,x).
  \]
  But this implies that
  \begin{dmath*}
    \pi(x) T(x,y)
    =
    \pi(x) \Exv{ T_{i,s}(x,y) }
    =
    \Exv{ \pi(x) T_{i,s}(x,y) }
    =
    \Exv{ \pi(y) T_{i,s}(y,x) }
    =
    \pi(y) \Exv{ T_{i,s}(y,x) }
    =
    \pi(y) T(y,x),
  \end{dmath*}
  so the whole chain is reversible.
  This is what we wanted to prove.
\end{proof}

\begin{proof}[Proof of Theorem~\ref{thmMGPMHgap}]
  As before, we will accomplish this proof via the technique of Dirichlet forms.
  First, recall that the transition probability matrix $T$ of vanilla Gibbs sampling is, for any $x$ and $y$ which differ in only variable $i$,
  \begin{dmath*}
    T(x, y)
    =
    \frac{1}{n} \frac{
      \exp(\zeta(y))
    }{
      \sum_{w = 1}^D \exp\left(\zeta(y_{i \rightarrow w}) \right)
    },
  \end{dmath*}
  where here $y_{i \rightarrow w}$ denotes the state $y$ with variable $i$ assigned to $w$; that is,
  \[
    y_{i \rightarrow w}
    =
    (x \cap y) \cup \{(i, w)\}.
  \] 
  Multiplying by $\pi(x)$,
  \begin{dmath*}
    \pi(x) T(x, y)
    =
    \frac{1}{n Z} \frac{
      \exp(\zeta(x)) \exp(\zeta(y))
    }{
      \sum_{w = 1}^D \exp\left(\zeta(y_{i \rightarrow w}) \right)
    },
  \end{dmath*}
  From the proof of Theorem~\ref{thmMGPMHreversible}, we have that the transition probability matrix of MGPMH (which we denote with $\bar T$) satisfies
  \begin{dmath*}
    \pi(x) \bar T(x,y)
    =
    \pi(x) \Exv[j,s]{ \bar T_{j,s}(x,y) }
    =
    \frac{1}{n} \pi(x) \Exv[s]{ \bar T_{i,s}(x,y) }
    =
    \frac{1}{n Z}
    \Exv{
      \frac{
        \min\left(
          \exp\left( \zeta(y) \right)
          \cdot
          \exp\left( \epsilon_{x(i)} \right)
          , 
          \exp\left( \zeta(x) \right)
          \cdot
          \exp\left( \epsilon_{y(i)} \right)
        \right)
      }{
        \sum_{w=1}^D \exp(\epsilon_w)
      }
    }.
  \end{dmath*}
  It follows that
  \begin{dmath*}
    \frac{
      \pi(x) \bar T(x,y)
    }{
      \pi(x) T(x, y)
    }
    =
    \frac{
      \frac{1}{n Z}
      \Exv{
        \frac{
          \min\left(
            \exp\left( \zeta(y) \right)
            \cdot
            \exp\left( \epsilon_{x(i)} \right)
            , 
            \exp\left( \zeta(x) \right)
            \cdot
            \exp\left( \epsilon_{y(i)} \right)
          \right)
        }{
          \sum_{w=1}^D \exp(\epsilon_w)
        }
      }
    }{
      \frac{1}{n Z} \frac{
        \exp(\zeta(x)) \exp(\zeta(y))
      }{
        \sum_{w = 1}^D \exp\left(\zeta(y_{i \rightarrow w}) \right)
      }
    }
    =
    \Exv{
      \frac{
        \sum_{w = 1}^D \exp\left(\zeta(y_{i \rightarrow w}) \right)
      }{
        \sum_{w=1}^D \exp(\epsilon_w)
      }
      \min\left(
        \exp\left( \epsilon_{x(i)} - \zeta(x) \right)
        , 
        \exp\left( \epsilon_{y(i)} - \zeta(y) \right)
      \right)
    }
    =
    \Exv{
      \frac{
        \sum_{w = 1}^D \exp\left(\zeta(y_{i \rightarrow w}) \right)
      }{
        \sum_{w=1}^D \exp(\epsilon_w)
      }
      \cdot
      \frac{1}{
        \max\left(
          \exp\left( \zeta(x) - \epsilon_{x(i)} \right)
          , 
          \exp\left( \zeta(y) - \epsilon_{y(i)} \right)
        \right)
      }
    }
    =
    \Exv{
      \frac{
        \sum_{w = 1}^D \exp\left(\zeta(y_{i \rightarrow w}) \right)
      }{
        \sum_{w=1}^D
        \max\left(
          \exp\left( \epsilon_w + \zeta(x) - \epsilon_{x(i)} \right)
          , 
          \exp\left( \epsilon_w + \zeta(y) - \epsilon_{y(i)} \right)
        \right)
      }
    }.
  \end{dmath*}
  By Jensen's inequality, since $f(z) = 1/z$ is convex, we can bound this from below by
  \begin{dmath*}
    \frac{
      \pi(x) \bar T(x,y)
    }{
      \pi(x) T(x, y)
    }
    \ge
    \frac{
      \sum_{w = 1}^D \exp\left(\zeta(y_{i \rightarrow w}) \right)
    }{
      \Exv{
        \sum_{w=1}^D
        \max\left(
          \exp\left( \epsilon_w + \zeta(x) - \epsilon_{x(i)} \right)
          , 
          \exp\left( \epsilon_w + \zeta(y) - \epsilon_{y(i)} \right)
        \right)
      }
    }
    =
    \frac{
      \sum_{w = 1}^D \exp\left(\zeta(y_{i \rightarrow w}) \right)
    }{
      \sum_{w=1}^D
      \exp\left(\zeta(y_{i \rightarrow w}) \right)
      \Exv{
        \max\left(
          \exp\left( \epsilon_w - \zeta(y_{i \rightarrow w}) + \zeta(x) - \epsilon_{x(i)} \right)
          , 
          \exp\left( \epsilon_w - \zeta(y_{i \rightarrow w}) + \zeta(y) - \epsilon_{y(i)} \right)
        \right)
      }
    }.
  \end{dmath*}
  Next, we notice that if we define $z = y_{i \rightarrow w}$ for a particular $w$, we can write
  \begin{dmath*}
    \Exv{
      \max\left(
        \exp\left( \epsilon_w - \zeta(y_{i \rightarrow w}) + \zeta(x) - \epsilon_{x(i)} \right)
        , 
        \exp\left( \epsilon_w - \zeta(y_{i \rightarrow w}) + \zeta(y) - \epsilon_{y(i)} \right)
      \right)
    }
    =
    \Exv{
      \max\left(
        \exp\left( \epsilon_{z(i)} - \zeta(z) + \zeta(x) - \epsilon_{x(i)} \right)
        , 
        \exp\left( \epsilon_{z(i)} - \zeta(z) + \zeta(y) - \epsilon_{y(i)} \right)
      \right)
    }
  \end{dmath*}
  Recall that the random variables in this expected value are the $\epsilon_w$, which are functions of the minibatch coefficients $s_{\phi}$ where
  \[
    \epsilon_{z(i)} = \sum_{\phi \in A[i]} \frac{s_{\phi} L}{\lambda M_{\phi}} \phi(z).
  \]
  So,
  \begin{dmath*}
    \Exv{
      \max\left(
        \exp\left( \epsilon_{z(i)} - \zeta(z) + \zeta(x) - \epsilon_{x(i)} \right)
        , 
        \exp\left( \epsilon_{z(i)} - \zeta(z) + \zeta(y) - \epsilon_{y(i)} \right)
      \right)
    }
    =
    \Exv{
      \max\left(
        \exp\left(
          \left( \sum_{\phi \in A[i]} \frac{s_{\phi} L}{\lambda M_{\phi}} \phi(z) \right)
          - 
          \zeta(z) 
          + 
          \zeta(x) 
          -
          \left( \sum_{\phi \in A[i]} \frac{s_{\phi} L}{\lambda M_{\phi}} \phi(x) \right)
        \right)
        ,\\
        \exp\left(
          \left( \sum_{\phi \in A[i]} \frac{s_{\phi} L}{\lambda M_{\phi}} \phi(z) \right)
          - 
          \zeta(z) 
          + 
          \zeta(y) 
          -
          \left( \sum_{\phi \in A[i]} \frac{s_{\phi} L}{\lambda M_{\phi}} \phi(y) \right)
        \right)
      \right)
    }
    =
    \Exv{
      \max\left(
        \exp\left(
          \sum_{\phi \in A[i]} \left( \frac{s_{\phi} L}{\lambda M_{\phi}} - 1 \right) (\phi(z) - \phi(x))
        \right)
        ,
        \exp\left(
          \sum_{\phi \in A[i]} \left( \frac{s_{\phi} L}{\lambda M_{\phi}} - 1 \right) (\phi(z) - \phi(y))
        \right)
      \right)
    }.
  \end{dmath*}
  Since the maximum of two sums is less than the sum of the maximums, it follows that
  \begin{dmath*}
    \Exv{
      \max\left(
        \exp\left( \epsilon_{z(i)} - \zeta(z) + \zeta(x) - \epsilon_{x(i)} \right)
        , 
        \exp\left( \epsilon_{z(i)} - \zeta(z) + \zeta(y) - \epsilon_{y(i)} \right)
      \right)
    }
    =
    \Exv{
      \exp\left(
        \max\left(
          \sum_{\phi \in A[i]} \left( \frac{s_{\phi} L}{\lambda M_{\phi}} - 1 \right) (\phi(z) - \phi(x))
          ,
          \sum_{\phi \in A[i]} \left( \frac{s_{\phi} L}{\lambda M_{\phi}} - 1 \right) (\phi(z) - \phi(y))
        \right)
      \right)
    }
    \le
    \Exv{
      \exp\left(
        \sum_{\phi \in A[i]} 
        \max\left(
          \left( \frac{s_{\phi} L}{\lambda M_{\phi}} - 1 \right) (\phi(z) - \phi(x))
          ,
          \left( \frac{s_{\phi} L}{\lambda M_{\phi}} - 1 \right) (\phi(z) - \phi(y))
        \right)
      \right)
    }
    =
    \Exv{
      \exp\left(
        \sum_{\phi \in A[i]} 
        \left( \frac{s_{\phi} L}{\lambda M_{\phi}} - 1 \right)
        \max\left(
          \phi(z) - \phi(x)
          ,
          \phi(z) - \phi(y)
        \right)
      \right)
    }
    \le
    \Exv{
      \exp\left(
        \sum_{\phi \in A[i]} 
        \left( \frac{s_{\phi} L}{\lambda M_{\phi}} - 1 \right) M_{\phi}
      \right)
    },
  \end{dmath*}
  where this last line follows from the fact that $\phi(z) - \phi(x) \le M_{\phi}$.
  So, by independence of the $s_{\phi}$,
  \begin{dmath*}
    \Exv{
      \max\left(
        \exp\left( \epsilon_{z(i)} - \zeta(z) + \zeta(x) - \epsilon_{x(i)} \right)
        , 
        \exp\left( \epsilon_{z(i)} - \zeta(z) + \zeta(y) - \epsilon_{y(i)} \right)
      \right)
    }
    \le
    \prod_{\phi \in A[i]} 
    \Exv{
      \exp\left(
        \left( \frac{s_{\phi} L}{\lambda M_{\phi}} - 1 \right) M_{\phi}
      \right)
    }
    =
    \prod_{\phi \in A[i]} 
    \exp(-M_\phi)
    \Exv{
      \exp\left(
        \frac{s_{\phi} L}{\lambda}
      \right)
    }.
  \end{dmath*}
  This last expression is just the moment generating function of the Poisson random variable, evaluated at $t = \frac{L}{\lambda}$.
  Since $s_{\phi}$ has parameter $\frac{\lambda M_{\phi}}{L}$, it follows from known properties of the Poisson distribution that
  \[
    \Exv{
      \exp\left(
        \frac{s_{\phi} L}{\lambda}
      \right)
    }
    =
    \exp\left(
      \frac{\lambda M_{\phi}}{L}
      \left(
        \exp\left( \frac{L}{\lambda} \right) - 1
      \right)
    \right).
  \]
  Multiplying both sides by $\exp(-M_{\phi})$,
  \[
    \exp(-M_\phi)
    \Exv{
      \exp\left(
        \frac{s_{\phi} L}{\lambda}
      \right)
    }
    =
    \exp\left(
      M_{\phi}
      \left(
        \frac{\lambda}{L}
        \left(
          \exp\left( \frac{L}{\lambda} \right) - 1
        \right)
        -
        1
      \right)
    \right).
  \]
  And substituing this back into our previous expression,
  \begin{dmath*}
    \Exv{
      \max\left(
        \exp\left( \epsilon_{z(i)} - \zeta(z) + \zeta(x) - \epsilon_{x(i)} \right)
        , 
        \exp\left( \epsilon_{z(i)} - \zeta(z) + \zeta(y) - \epsilon_{y(i)} \right)
      \right)
    }
    \le
    \prod_{\phi \in A[i]} 
    \exp\left(
      M_{\phi}
      \left(
        \frac{\lambda}{L}
        \left(
          \exp\left( \frac{L}{\lambda} \right) - 1
        \right)
        -
        1
      \right)
    \right)
    =
    \exp\left(
      \sum_{\phi \in A[i]} 
      M_{\phi}
      \left(
        \frac{\lambda}{L}
        \left(
          \exp\left( \frac{L}{\lambda} \right) - 1
        \right)
        -
        1
      \right)
    \right)
    =
    \exp\left(
      L
      \left(
        \frac{\lambda}{L}
        \left(
          \exp\left( \frac{L}{\lambda} \right) - 1
        \right)
        -
        1
      \right)
    \right)
    =
    \exp\left(
      \lambda
      \left(
        \exp\left( \frac{L}{\lambda} \right) - 1
      \right)
      -
      L
    \right).
  \end{dmath*}
  As long as $z \le 1$, we know that $\exp(z) - 1 \le z + z^2$.
  So, as long as $L \le \lambda$, it follows that
  \begin{dmath*}
    \Exv{
      \max\left(
        \exp\left( \epsilon_{z(i)} - \zeta(z) + \zeta(x) - \epsilon_{x(i)} \right)
        , 
        \exp\left( \epsilon_{z(i)} - \zeta(z) + \zeta(y) - \epsilon_{y(i)} \right)
      \right)
    }
    \le
    \exp\left(
      \lambda
      \left(
        \frac{L}{\lambda} + \left( \frac{L}{\lambda} \right)^2
      \right)
      -
      L
    \right)
    =
    \exp\left(
      \frac{L^2}{\lambda}
    \right).
  \end{dmath*}
  Substituting this back into our previous expression above, we have that
  \begin{dmath*}
    \frac{
      \pi(x) \bar T(x,y)
    }{
      \pi(x) T(x, y)
    }
    \ge
    \frac{
      \sum_{w = 1}^D \exp\left(\zeta(y_{i \rightarrow w}) \right)
    }{
      \sum_{w=1}^D
      \exp\left(\zeta(y_{i \rightarrow w}) \right)
      \Exv{
        \max\left(
          \exp\left( \epsilon_w - \zeta(y_{i \rightarrow w}) + \zeta(x) - \epsilon_{x(i)} \right)
          , 
          \exp\left( \epsilon_w - \zeta(y_{i \rightarrow w}) + \zeta(y) - \epsilon_{y(i)} \right)
        \right)
      }
    }
    \ge
    \frac{
      \sum_{w = 1}^D \exp\left(\zeta(y_{i \rightarrow w}) \right)
    }{
      \sum_{w=1}^D
      \exp\left(\zeta(y_{i \rightarrow w}) \right)
      \cdot
      \exp\left(
        \frac{L^2}{\lambda}
      \right)
    }
    =
    \exp\left(
      -\frac{L^2}{\lambda}
    \right).
  \end{dmath*}
  Thus, it follows that
  \[
    \frac{
      \pi(x) \bar T(x,y)
    }{
      \pi(x) T(x, y)
    }
    \ge
    \exp\left(
      -\frac{L^2}{\lambda}
    \right).
  \]
  Now, we finally use the Dirichlet forms.
  By definition,
  \begin{dmath*}
    \bar \gamma
    =
    \min_f
    \frac{
      \bar{\mathcal{E}}(f)
    }{
      \Var[\pi]{f}
    }
    =
    \min_f
    \frac{
      1
    }{
      \Var[\pi]{f}
    }
    \sum_{x,y} \pi(x) \bar T(x, y)
    \ge
    \exp\left(
      -\frac{L^2}{\lambda}
    \right)
    \cdot
    \min_f
    \frac{
      1
    }{
      \Var[\pi]{f}
    }
    \sum_{x,y} \pi(x) T(x, y)
    =
    \exp\left(
      -\frac{L^2}{\lambda}
    \right)
    \cdot
    \min_f
    \frac{
      \mathcal{E}(f)
    }{
      \Var[\pi]{f}
    }
    =
    \exp\left(
      -\frac{L^2}{\lambda}
    \right)
    \cdot
    \gamma.
  \end{dmath*}
  This is what we wanted to prove.
\end{proof}

\begin{proof}[Proof of Theorem~\ref{thmDoubleMINTreversible}]
The probability of transitioning from $(x, \xi_x)$ to $(y, \xi_y)$ for two distinct states $x$ and $y$ which differ only in variable $i$ will be the probability that: we decide to re-sample variable $i$, we choose $y$ as our proposal, we sample $\xi_y$ as the energy for state $y$, and we accept the proposed change.
We can write this transition probability as
\begin{dmath*}
  T((x, \xi_x), (y, \xi_y))
  =
  \frac{1}{n}
  \Exv{
    \psi(y)
    \cdot
    \mu_y(\xi_y)
    \cdot
    \min(a, 1)
  },
\end{dmath*}
where here the expected value is taken over the random minibatch coefficients $s_\phi$.
We can expand this out to
\begin{dmath*}
  T((x, \xi_x), (y, \xi_y))
  =
  \frac{1}{n}
  \Exv{
    \frac{\exp(\epsilon_{y(i)})}{\sum_{u=1}^D \exp(\epsilon_u)}
    \cdot
    \mu_y(\xi_y)
    \cdot
    \min\left(
      \frac{
        \exp(\xi_y)
      }{
        \exp(\xi_x)
      }
      \cdot
      \frac{
        \exp\left(
          \epsilon_{x(i)}
        \right)
      }{
        \exp\left(
          \epsilon_{y(i)}
        \right)
      }
    , 1 \right)
  }.
\end{dmath*}
If we define stationary distribution
\[
  \pi(x, \xi)
  =
  \frac{1}{Z}
  \mu_x(\xi) \cdot \exp(\xi), 
\]
then
\begin{dmath*}
  \pi(x, \xi_x)
  \cdot
  T((x, \xi_x), (y, \xi_y))
  =
  \frac{1}{n Z}
  \Exv{
    \mu_x(\xi_x) \cdot \exp(\xi_x)
    \cdot
    \frac{\exp(\epsilon_{y(i)})}{\sum_{u=1}^D \exp(\epsilon_u)}
    \cdot
    \mu_y(\xi_y)
    \cdot
    \min\left(
      \frac{
        \exp(\xi_y)
      }{
        \exp(\xi_x)
      }
      \cdot
      \frac{
        \exp\left(
          \epsilon_{x(i)}
        \right)
      }{
        \exp\left(
          \epsilon_{y(i)}
        \right)
      }
    , 1 \right)
  }
  =
  \frac{1}{n Z}
  \Exv{
    \frac{
      \mu_x(\xi_x)
      \cdot
      \mu_y(\xi_y)
    }{\sum_{u=1}^D \exp(\epsilon_u)}
    \cdot
    \min\left(
      \exp(\xi_y)
      \cdot
      \exp(\epsilon_{x(i)})
    , 
      \exp(\xi_x)
      \cdot
      \exp(\epsilon_{y(i)})
    \right)
  }.
\end{dmath*}
This expression is clearly symmetric in $(x, \xi_x)$ and $(y, \xi_y)$, so it follows that the chain is reversible with stationary distribution $\pi$ as defined above.
This is what we wanted to prove.
\end{proof}

\begin{proof}[Proof of Theorem~\ref{thmDoubleMINTgap}]
As with the analysis of MIN-Gibbs, we will prove this result using the technique of Dirichlet forms.
From the result of Theorem~\ref{thmDoubleMINTreversible}, we know that for some $\bar Z$, the stationary distribution of DoubleMIN-Gibbs is
\[
  \bar \pi(x, \xi) = \frac{1}{\bar Z} \mu_x(\xi) \cdot \exp(\xi).
\]
We also know from that same theorem that the transition probability matrix can be written as
\begin{dmath*}
  \bar \pi(x, \xi_x)
  \cdot
  \bar T((x, \xi_x), (y, \xi_y))
  =
  \frac{1}{n \bar Z}
  \Exv{
    \frac{
      \mu_x(\xi_x)
      \cdot
      \mu_y(\xi_y)
    }{\sum_{u=1}^D \exp(\epsilon_u)}
    \cdot
    \min\left(
      \exp(\xi_y)
      \cdot
      \exp(\epsilon_{x(i)})
    , 
      \exp(\xi_x)
      \cdot
      \exp(\epsilon_{y(i)})
    \right)
  }.
\end{dmath*}
It follows from the same analysis as in the proof of Theorem~\ref{thmMINTreversible} that the Dirichlet form of DoubleMIN-Gibbs is
\begin{dmath*}
  \bar{\mathcal{E}}(f)
  =
  \frac{1}{2 n \bar Z}
  \sum_{(x, y) \in Q}
  \Exv{
    \left(f(x, \xi_x) - f(y, \xi_y) \right)^2
    \cdot
    \frac{
      \min\left(
        \exp(\xi_y)
        \cdot
        \exp(\epsilon_{x(i)})
      , 
        \exp(\xi_x)
        \cdot
        \exp(\epsilon_{y(i)})
      \right)
    }{
      \sum_{u=1}^D \exp(\epsilon_u)
    }
  },
\end{dmath*}
where $Q \subset \Omega \times \Omega$ as before is the set of pairs of stats which differ only in a single variable.
Here the expected value is also taken over random variables $\xi_x$ and $\xi_y$, which we suppose are sampled independently from $\mu_x$ and $\mu_y$, respectively.
By a similar argument, we can also determine that the MGPMH chain will have Dirichlet form
\begin{dmath*}
  \bar{\mathcal{E}}(g)
  =
  \frac{1}{2 n \bar Z}
  \sum_{(x, y) \in Q}
  \Exv{
    \left(g(x) - g(y) \right)^2
    \cdot
    \frac{
      \min\left(
        \exp(\zeta(y))
        \cdot
        \exp(\epsilon_{x(i)})
      , 
        \exp(\zeta(x))
        \cdot
        \exp(\epsilon_{y(i)})
      \right)
    }{
      \sum_{u=1}^D \exp(\epsilon_u)
    }
  }.
\end{dmath*}
Now, we know by the condition of the theorem that
\[
  \Abs{\xi_x - \zeta(x)} \le \delta.
\]
Therefore, we can bound the Dirichlet form of DoubleMIN-Gibbs from below with
\begin{dmath*}
  \bar{\mathcal{E}}(f)
  \ge
  \frac{1}{2 n \bar Z}
  \sum_{(x, y) \in Q}
  \Exv{
    \left(f(x, \xi_x) - f(y, \xi_y) \right)^2
    \cdot
    \frac{
      \min\left(
        \exp(\zeta(y) - \delta)
        \cdot
        \exp(\epsilon_{x(i)})
      , 
        \exp(\zeta(x) - \delta)
        \cdot
        \exp(\epsilon_{y(i)})
      \right)
    }{
      \sum_{u=1}^D \exp(\epsilon_u)
    }
  }
  =
  \frac{\exp(-\delta)}{2 n \bar Z}
  \sum_{(x, y) \in Q}
  \Exv{
    \left(f(x, \xi_x) - f(y, \xi_y) \right)^2
    \cdot
    \frac{
      \min\left(
        \exp(\zeta(y))
        \cdot
        \exp(\epsilon_{x(i)})
      , 
        \exp(\zeta(x))
        \cdot
        \exp(\epsilon_{y(i)})
      \right)
    }{
      \sum_{u=1}^D \exp(\epsilon_u)
    }
  }
  =
  \frac{\exp(-\delta)}{2 n \bar Z}
  \sum_{(x, y) \in Q}
  \Exv{
    \left(f(x, \xi_x) - f(y, \xi_y) \right)^2
  }
  \Exv{
    \frac{
      \min\left(
        \exp(\zeta(y))
        \cdot
        \exp(\epsilon_{x(i)})
      , 
        \exp(\zeta(x))
        \cdot
        \exp(\epsilon_{y(i)})
      \right)
    }{
      \sum_{u=1}^D \exp(\epsilon_u)
    }
  },
\end{dmath*}
where we can separate out the expected values like this because $\xi_x$ and $\xi_y$ are sampled independently from the other random variables $s_\phi$.
On the other hand, the variance form associated with the DoubleMIN-Gibbs chain is
\begin{dmath*}
  \Var[\bar \pi]{f}
  =
  \frac{1}{2 \bar Z^2} \sum_{x \in \Omega} \sum_{y \in \Omega}
  \Exv{
    \left(f(x, \xi_x) - f(y, \xi_y) \right)^2
    \cdot
    \exp(\xi_x) \cdot \exp(\xi_y)
  }
  \le
  \frac{1}{2 \bar Z^2} \sum_{x \in \Omega} \sum_{y \in \Omega}
  \Exv{
    \left(f(x, \xi_x) - f(y, \xi_y) \right)^2
    \cdot
    \exp(\zeta(x) + \delta) \cdot \exp(\zeta(y) + \delta)
  }
  =
  \frac{\exp(2 \delta)}{2 \bar Z^2} \sum_{x \in \Omega} \sum_{y \in \Omega}
  \exp(\zeta(x)) \cdot \exp(\zeta(y))
  \Exv{
    \left(f(x, \xi_x) - f(y, \xi_y) \right)^2
  }.
\end{dmath*}
From here, we can write
\begin{dmath*}
  \bar \gamma
  =
  \min_f
  \frac{
    \bar{\mathcal{E}}(f)
  }{
    \Var[\bar \pi]{f}
  }
  \ge
  \min_f
  \frac{
    \frac{\exp(-\delta)}{2 n \bar Z}
    \sum_{(x, y) \in Q}
    \Exv{
      \frac{
        \min\left(
          \exp(\zeta(y))
          \cdot
          \exp(\epsilon_{x(i)})
        , 
          \exp(\zeta(x))
          \cdot
          \exp(\epsilon_{y(i)})
        \right)
      }{
        \sum_{u=1}^D \exp(\epsilon_u)
      }
    }
    \Exv{
      \left(f(x, \xi_x) - f(y, \xi_y) \right)^2
    }
  }{
    \frac{\exp(2 \delta)}{2 \bar Z^2} \sum_{x \in \Omega} \sum_{y \in \Omega}
    \exp(\zeta(x)) \cdot \exp(\zeta(y))
    \Exv{
      \left(f(x, \xi_x) - f(y, \xi_y) \right)^2
    }
  }
  =
  \frac{\exp(-3 \delta) \bar Z}{n}
  \min_f
  \frac{
    \sum_{(x, y) \in Q}
    \Exv{
      \frac{
        \min\left(
          \exp(\zeta(y))
          \cdot
          \exp(\epsilon_{x(i)})
        , 
          \exp(\zeta(x))
          \cdot
          \exp(\epsilon_{y(i)})
        \right)
      }{
        \sum_{u=1}^D \exp(\epsilon_u)
      }
    }
    \Exv{
      \left(f(x, \xi_x) - f(y, \xi_y) \right)^2
    }
  }{
    \sum_{x \in \Omega} \sum_{y \in \Omega}
    \exp(\zeta(x)) \cdot \exp(\zeta(y))
    \Exv{
      \left(f(x, \xi_x) - f(y, \xi_y) \right)^2
    }
  }.
\end{dmath*}
As before, using the fact that for positive numbers $a_1, a_2, \ldots, a_N$ and $b_1, b_2, \ldots, b_N$,
\[
  \frac{\sum_{i=1}^n a_i}{\sum_{i=1}^n b_i} \ge \min_i \frac{a_i}{b_i}.
\]
we can bound this from below for some fixed $\xi_x$ and $\xi_y$ (which may still be a function of $f$) with
\begin{dmath*}
  \bar \gamma
  \ge
  \frac{\exp(-3 \delta) \bar Z}{n}
  \min_f
  \frac{
    \sum_{(x, y) \in Q}
    \Exv{
      \frac{
        \min\left(
          \exp(\zeta(y))
          \cdot
          \exp(\epsilon_{x(i)})
        , 
          \exp(\zeta(x))
          \cdot
          \exp(\epsilon_{y(i)})
        \right)
      }{
        \sum_{u=1}^D \exp(\epsilon_u)
      }
    }
    \left(f(x, \xi_x) - f(y, \xi_y) \right)^2
  }{
    \sum_{x \in \Omega} \sum_{y \in \Omega}
    \exp(\zeta(x)) \cdot \exp(\zeta(y))
    \left(f(x, \xi_x) - f(y, \xi_y) \right)^2
  }
  \ge
  \frac{\exp(-3 \delta) \bar Z}{n}
  \min_g
  \frac{
    \sum_{(x, y) \in Q}
    \Exv{
      \frac{
        \min\left(
          \exp(\zeta(y))
          \cdot
          \exp(\epsilon_{x(i)})
        , 
          \exp(\zeta(x))
          \cdot
          \exp(\epsilon_{y(i)})
        \right)
      }{
        \sum_{u=1}^D \exp(\epsilon_u)
      }
    }
    \left(g(x) - g(y) \right)^2
  }{
    \sum_{x \in \Omega} \sum_{y \in \Omega}
    \exp(\zeta(x)) \cdot \exp(\zeta(y))
    \left(g(x) - g(y) \right)^2
  },
\end{dmath*}
where this last inequality holds because we can always set
\[
  g(x) = f(x, \xi_x).
\]
Finally, we can rewrite this as
\begin{dmath*}
  \bar \gamma
  \ge
  \frac{\exp(-3 \delta) \bar Z}{Z}
  \min_g
  \frac{
    \frac{1}{2 n Z}
    \sum_{(x, y) \in Q}
    \Exv{
      \frac{
        \min\left(
          \exp(\zeta(y))
          \cdot
          \exp(\epsilon_{x(i)})
        , 
          \exp(\zeta(x))
          \cdot
          \exp(\epsilon_{y(i)})
        \right)
      }{
        \sum_{u=1}^D \exp(\epsilon_u)
      }
    }
    \left(g(x) - g(y) \right)^2
  }{
    \frac{1}{2 Z^2}
    \sum_{x \in \Omega} \sum_{y \in \Omega}
    \exp(\zeta(x)) \cdot \exp(\zeta(y))
    \left(g(x) - g(y) \right)^2
  }
  =
  \frac{\exp(-3 \delta) \bar Z}{Z}
  \min_g
  \frac{
    \mathcal{E}(g)
  }{
    \Var[\pi]{g}
  }
  =
  \frac{\exp(-3 \delta) \bar Z}{Z}
  \gamma.
\end{dmath*}
All that remains is to bound the ratio of the $Z$ and $\bar Z$.
But this ratio is the same as it is in the proof of Theorem~\ref{thmMINTgap}, so we can conclude that
\[
  \frac{\bar Z}{Z}
  \ge
  \exp(-\delta).
\]
So,
\[
  \bar \gamma \ge \exp(-4 \delta) \gamma.
\]
This is what we wanted to show.
\end{proof}

\section{Experiments}

In this section, we describe the methodology of our experiments in more detail.
Our experiments are run on a synthetic Ising model with energy
\[
  \zeta_{\text{Ising}}(x) = \sum_{i = 1}^N \sum_{j = 1}^N \beta \cdot A_{ij} \cdot (x(i) x(j) + 1),
\]
and a synthetic Potts model with energy
\[
  \zeta_{\text{Potts}} = \sum_{i = 1}^N \sum_{j = 1}^N \beta \cdot A_{ij} \cdot \delta(x(i), x(j)).
\]
As is usually the case with these models, we laid out our variables on a grid, a $20 \times 20$ grid to be precise.
This resulted in $n = 400$ variables for both models.
Our goal here was to construct a dense synthetic model with a nontrivial interaction matrix (i.e. non-constant $A_{ij}$).
To do this, we assigned $A_{ij}$ (for $i \ne j$) according to the Gaussian kernel
\[
  A_{ij} = \exp\left(-\gamma d_{ij}^2 \right)
\]
where $d_{ij}$ is the distance between variables $i$ and $j$ in the $20 \times 20$ grid.
Equivalently, if $x_i \in \{1, \ldots, 20\} \times \{1, \ldots, 20\}$ is the position of variable $i$ in the grid, then
\[
  A_{ij} = \exp\left(-\gamma \|\mathbf x_i - \mathbf x_j\|^2 \right).
\]
This form may be more easily recognizable as a Gaussian RBF kernel.
We chose to set $\gamma = 1.5$ for both the Ising and Potts models.
For all experiments, we ran $1000000 = 10^6$ iterations of sampling.
For the Ising model, we set the inverse temperature $\beta = 1.0$, and for the Potts model, we set $\beta = 4.6$.
In both cases, we hand-tuned $\beta$ so that it was small enough that the marginal error of vanilla Gibbs sampling would clearly be observed to converge within $10^6$ iterations, but large enough that its convergence trajectory would be clearly distinct from the trivial convergence trajectory when $\beta = 0$.
In other words, we set $\beta$ so that, for plain Gibbs sampling, we would both observe non-trivial behavior and have a chain that converged fast enough to efficiently simulate.
We then used this value of $\beta$ to evaluate our algorithms.

\end{document}